\documentclass{article}
\usepackage[preprint]{icml2026}

\usepackage[utf8]{inputenc} %
\usepackage[T1]{fontenc} 
\usepackage{microtype}
\usepackage{amsmath}
\allowdisplaybreaks
\usepackage{hyperref}
\hypersetup{hidelinks}
\usepackage{amsfonts}
\usepackage{natbib}
\usepackage{booktabs}
\usepackage{amsthm}
\usepackage{url}
\usepackage{graphicx} %
\usepackage[dvipsnames,svgnames]{xcolor}         %
\usepackage{comment}
\usepackage{enumitem}
\usepackage{wrapfig}

\newcommand{\R}{\mathbb{R}}
\newcommand{\PP}{\mathbb{P}}
\newcommand{\E}{\mathbb{E}}

\newcommand{\cD}{\mathcal{D}}

\newcommand{\cN}{\mathcal N}

\newcommand{\cR}{\mathcal R}
\newcommand{\tr}{\textnormal{tr}}
\newcommand{\orr}[1]{\overrightarrow{#1}}
\newcommand{\orl}[1]{\overleftarrow{#1}}
\newcommand{\fwdX}[1]{\orr{X_{#1}}}

\newcommand{\bwdX}[1]{\orl{X}_{#1}}
\newcommand{\bwdhatX}[1]{\orl{\hat{X}}_{#1}}
\newcommand{\bwdtildeX}[1]{\orl{\Tilde{X}}_{#1}}

\theoremstyle{plain}
\newtheorem{theorem}{Theorem}[section]

\newtheorem{proposition}[theorem]{Proposition}
\newtheorem{corollary}[theorem]{Corollary}
\newtheorem{lemma}[theorem]{Lemma}

\DeclareMathOperator*{\argmin}{arg\,min}

\definecolor{ForestGreen}{cmyk}{0.91,0,0.88,0.12}
\colorlet{pierrem}{ForestGreen}

\colorlet{claireb}{RoyalBlue}

\definecolor{Blue}{rgb}{0.0,0.72,0.92}
\colorlet{romu}{Blue}

\icmltitlerunning{Optimal Stopping in Latent Diffusion Models}

\begin{document}
\twocolumn[
\icmltitle{Optimal Stopping in Latent Diffusion Models}
 \begin{icmlauthorlist}
    \icmlauthor{Yu-Han Wu}{lpsm,gdm}
    \icmlauthor{Quentin Berthet}{gdm}
    \icmlauthor{G\'erard Biau}{lpsm}
    \icmlauthor{Claire Boyer}{lmo}
    \icmlauthor{Romuald Elie}{gdm}
    \icmlauthor{Pierre Marion}{inria}
  \end{icmlauthorlist}
  \icmlaffiliation{lpsm}{LPSM, Sorbonne Universit\'e}
  \icmlaffiliation{gdm}{Google DeepMind}
  \icmlaffiliation{lmo}{LMO, Universit\'e Paris-Saclay}
  \icmlaffiliation{inria}{INRIA, ENS-PSL, Paris}
  \icmlcorrespondingauthor{Yu-Han Wu}{yhwu@google.com}
  \icmlcorrespondingauthor{Pierre Marion}{pierre.marion@inria.fr}
  \icmlkeywords{Machine Learning, ICML}

  \vskip 0.3in
]
\printAffiliationsAndNotice{}
\begin{abstract}

We identify and analyze a surprising phenomenon of \textit{Latent} Diffusion Models (LDMs) where the final steps of the diffusion can \textit{degrade} sample quality. In contrast to conventional arguments that justify early stopping for numerical stability, this phenomenon is intrinsic to the dimensionality reduction in LDMs. We provide a principled explanation by analyzing the interaction between latent dimension and stopping time. Under a Gaussian framework with linear autoencoders, we characterize the conditions under which early stopping is needed to minimize the distance between generated and target distributions. More precisely, we show that lower-dimensional representations benefit from earlier termination, whereas higher-dimensional latent spaces require later stopping time. We further establish that the latent dimension interplays with other hyperparameters of the problem such as constraints in the parameters of score matching. Crucially, this framework suggests that the reconstruction quality of the autoencoder alone can serve as a proxy to estimate the potential performance of the full LDM. Experiments on synthetic and real datasets illustrate these properties, underlining that early stopping can improve generative quality. Together, our results offer a theoretical foundation for understanding how the latent dimension influences the sample quality, and highlight stopping time as a key hyperparameter in LDMs.
\end{abstract}

\section{Introduction}
\label{sec:intro} 
A pivotal advancement in the evolution of diffusion models is the introduction of the Latent Diffusion Model \citep[LDM,][]{rombach2022high}. 
Instead of performing the computationally intensive diffusion process in the high-dimensional pixel space, LDMs first compress the data into a lower-dimensional latent space using a pretrained autoencoder \citep[AE,][]{kingma2013auto}. The diffusion steps then occur within this more manageable latent representation, significantly reducing computational requirements and training time without a meaningful loss of quality. %
Once the generative process is complete, a decoder maps the resulting latent vector back into a full-resolution image. %
The standard training pipeline for LDMs operates as a two-stage process: an AE is trained first to compress the data, followed by the training of a diffusion model within the resulting latent space.

While recent research has focused on optimizing AE training to yield better latent representations for diffusion \citep[e.g.,][]{kouzelis2025eq, zhou2025alias}, selecting the appropriate latent dimension remains a critical challenge. 
The dimensionality must be fixed during the initial AE training phase, yet identifying the dimension that yields the optimal tradeoff between generation quality and compute cost typically requires the resource-intensive process of training a subsequent LDM for each candidate.

Furthermore, distinct behavioral differences exist between pixel-space and latent-space diffusion. 
As shown in Figure~\ref{fig:intro-comparision-images}, latent-generated images appear to visibly stabilize in the final diffusion steps, in contrast to pixel-space diffusion, where the final steps are crucial for removing remaining noise. 
This challenges the common assumption that LDMs yield the best samples at the final timestep, suggesting instead that early stopping may actually improve image quality, as the decoder may introduce high-frequency artifacts in the final stages.

\begin{figure}
    \centering
    \includegraphics[width=\linewidth]{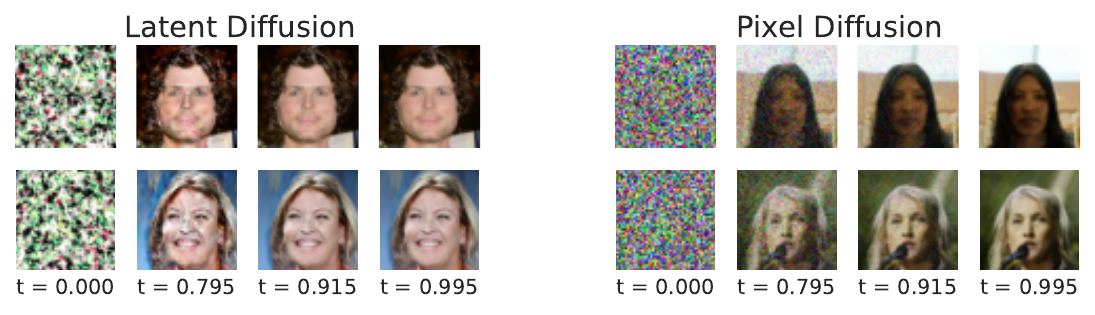}
    \caption{Samples generated with a latent diffusion model (LDM) and a pixel-space diffusion. In the LDM, the before-last sample is nearly denoised and indistinguishable from the final one, whereas in the pixel-space model stronger  noise remains at that timestep. See Appendix \ref{app:exp} for more examples.}
    \label{fig:intro-comparision-images} 
\end{figure}

\paragraph{Contributions and organization.} Our goal in this paper is to investigate the optimal stopping time and latent dimensionality for LDMs. More specifically, our contributions are as follows:
\begin{itemize}[left=2.5em]
    \item Working within a Gaussian framework, we explicitly derive the Wasserstein-2 distance, which is equivalent to the Fr\'echet distance in the Gaussian setting, between the true data distribution and the backward diffusion. 
    In particular, we highlight that the diffusion process is identical to injecting noise into the latent representations prior to decoding (Section \ref{sec:notations}).
    \item We determine the optimal latent dimension and stopping time for diagonal and general covariances. In particular, for data on a linear subspace, we prove that projecting onto that subspace is optimal. This reveals a trade-off: early generation steps favor lower dimensions, while later steps require higher dimensions for faithful reconstruction (Section \ref{sec:main-result}).
    \item When the score is learned by a restricted class of parametrized models (i.e., when the weights of the model are capped), we further establish the existence of an optimal latent projection that optimizes the backward diffusion, and we investigate how its dimension depends on the model class constraint and the data covariance (Section \ref{sec:erm}).
    \item Our theoretical analysis suggests the existence of an optimal interval for each LDM, which can be identified solely by examining the corresponding noisy AEs. Building on this insight, we conduct experiments on ImageNet-256 to demonstrate that these results also have predictive power in practical settings, i.e., we are able to predict the optimal stopping time and optimal interval for different LDMs by the FID curves of the corresponding version of noisy AEs. We train a diverse set of AEs and their corresponding LDMs, ensuring a controlled comparison by maintaining a consistent total parameter count across models. We then show that the Fr\'echet inception distance (FID) of the decoded noisy latent representations closely tracks that of the backward latent diffusion process. Notably, the FID curve of different LDMs cross at the same time as their corresponding version of noisy AEs (Section~\ref{sec:experiment}).
\end{itemize}

\section{Related work}
\label{sec:related-work}
\paragraph{Learning low-dimensional data with diffusion models.} Riemannian Diffusion Models, introduced by \citet{huang2022riemannian, de2022riemannian}, generalize the diffusion process to operate on Riemannian manifolds and preserve a known geometric structure by design. Subsequent theoretical work has analyzed the behavior of standard Denoising Diffusion Probabilistic Models (DDPMs) under the manifold hypothesis, demonstrating that they can implicitly adapt to the data's intrinsic dimension without explicit knowledge of the manifold \citep{tang2024adaptivity, george2025analysis}. 

Further improvements in computational and memory efficiency were introduced by LDMs \citep{rombach2022high} by first training a compression model to transform images into a  lower-dimensional latent space, from which the original data can be reconstructed at high fidelity. In practice, this approach is implemented with a regularized VAE \citep{esser2021taming}. The LDM is then trained in the latent space. Building on this core concept, LDMs have been extended to new domains, such as the generation of high-resolution videos \citep{blattmann2023align}. Furthermore, extensive research has focused on improving LDM's sampling quality, including methods like aligning encoded images with DINOv2 representations \citep{yu2024representation}, and enhancing the robustness of the latent space through explicit or implicit equivariance constraints \citep{kouzelis2025eq, skorokhodov2025improving, zhou2025alias}.
In contrast to standard diffusion models, theoretical properties of LDMs have been little studied; in this work, we investigate the connection of the latent dimension with diffusion stopping time and score matching regularization.
\paragraph{Optimal stopping time of diffusion models.} 
Focusing on a theoretical analysis of this phenomenon, \citet{achilli2025memorization} investigate the optimal stopping time for diffusion models under the assumption that the data is concentrated on a low-dimensional manifold, a concept formalized by the Hidden Manifold Model \citep{goldt2020modeling}. Closer to our contribution is the work of \citet{hurault2025score}. They also investigate the scenario where the true data distribution is Gaussian. Their analysis focuses on learning the score function using SGD, and allows them to determine an optimal stopping time. However, the study of these authors is limited to the diffusion model and did not consider the two-stage architecture of LDMs. Furthermore, the relationship between the data dimension and the derived optimal stopping time remained unexplored in their findings. In contrast, our work directly investigates the influence of the latent dimension on the optimal stopping time by incorporating an autoencoder into the diffusion model framework. We also demonstrate the need of early stopping without discretization of the backward diffusion process.
Finally, recent findings suggest that full decoding (at $t=0$) does not always yield the lowest FID \citep[e.g., ][]{jayasumana2024rethinking}. While this previous work attributed this behavior to be an artifact of the FID metric, in this work we challenge the common assumption that LDMs yield the best samples at the final timestep, suggesting instead that early stopping may actually improve image quality.

\section{Notations and problem setup}
\label{sec:notations}
This section introduces the mathematical formalism of diffusion models in the considered setting.

\paragraph{Latent Diffusion Models.} Let $p_0$ be an unknown distribution in $\R^D$. With a slight abuse of notation, we use in the following the same notation for a distribution and its density function. %
The goal of diffusion models is to generate new observations following $p_0$, given an i.i.d.~sample $(X_1, \hdots, X_n)$ drawn from $p_0$. The mechanism is as follows. Given a final diffusion time $T>0$, a latent dimension $d \leq D$, a semi-orthogonal projection matrix $P \in \R^{d \times D}$, and a scalar function $w:~[0, T]\to~\R$, the latent forward variance-preserving (VP)-SDE \citep{song2020score} is defined by
\begin{equation}
\label{eq:SDE-forward}
d P \fwdX t = -w_t^2 P \fwdX tdt + \sqrt{2w_t^2}d P \orr{W_t},\quad P\fwdX 0\sim P_\#p_0,
\end{equation}
where $\orr{W_t}$ is a standard $D$-dimensional Brownian motion. The role of the matrix $P$ is to perform linear dimension reduction. Two special cases are of interest: first, if $d=D$ and $P$ is the identity matrix, we recover the standard formulation of diffusion models. Second, if $P$ projects on the first few principal components of the sample covariance matrix, this amounts to performing principal component analysis \citep[PCA,][]{jolliffe2002principal}. This projection is equivalent to linear autoencoders \citep{plaut2018principal}, and there exists a pseudo-inverse $P^+\in\R^{D\times d}$ which allows us to map sample back to $\R^D$. We leave to future work to extend the analysis of this phenomenon to other noise schedules such as EDMs \citep{karras2022elucidating}.

Letting $s_P$ be the score function of $P\fwdX t$, i.e., $s_P(x,t) = \nabla\log p_P(x, t)$ where $p_P (\cdot, t)$ is the density function of $P\fwdX t$, the forward diffusion can be reversed in time using the backward process
\begin{align}
\label{eq:SDE-backward}
d P \bwdX t &= (w_{T-t}^2 P \bwdX t + 2w_{T-t}^2s_P(P \bwdX t, T-t))dt\nonumber \\
&\quad+ \sqrt{2w_{T-t}^2}dP\orl{W_t}, \quad P \bwdX 0 \sim P_\# p_T,
\end{align}
where $\orl{W_t}$ is a standard $D$-dimensional Brownian motion. This means that the marginal distribution of $P\bwdX{T-t}$ matches the marginal distribution of $P\fwdX t$ \citep{anderson1982reverse}. %
Hence running the backward diffusion allows to generate a sample from $\bwdX{T} \sim P_\#p_0$,  and then the pseudo-inverse $P^+$ can be used to map the generated sample back to $\R^D$. Importantly, this procedure requires knowledge of $s_P$, which can be estimated using the training sample.

\paragraph{Problem setup.} We assume that the data distribution $p_0$ is a centered $D$-dimensional Gaussian with independent components, i.e., 
\begin{equation}
\label{eq:independent-gaussian}
p_0=\cN(0, \Sigma)
\end{equation}

This specific setting simplifies our study but shall still provide important insights for more general distributions. We consider a hierarchy of latent spaces with increasing dimension $d$ from $1$ to $D$. This is formalized by defining the matrix $P$ in the VP-SDE \eqref{eq:SDE-forward} as the semi-orthogonal projection matrix $P_dO^\top$ where $P_d$ is the projection matrix onto the first $d$ dimensions and $O$ an orthonormal matrix. In this Gaussian framework, the marginal distributions of the backward process are Gaussian with an explicit covariance matrix:
\begin{lemma}
The distribution of $P_dO^\top \bwdX {T-t}$ is 
\begin{align}
\label{eq:add-noise}
P_dO^\top \bwdX{T-t} &\overset{d}{=}a_tP_dO^\top\fwdX0 + b_tZ\\
&\sim\cN(0, a_t^2I_d+b_t^2P_d O^\top\Sigma OP_d^\top)\,,\nonumber
\end{align}
where $Z\sim\cN(0, I_d), a_t=\sqrt{1-b_t^2}$ and $b_t=e^{-\int_0^tw_t^2dt}$, and its score is
\[
s_{P_d}(x, t) = -(a_t^2I_d+b_t^2P_d\Sigma P_d^\top)^{-1}x,\quad x\in\R^d.
\] 
\end{lemma}
Here, the pseudo-inverse of $P$ is $P^+ = OP_d^\top$, and the decoded backward distribution in the original data space is therefore given by $Q\bwdX {T-t} \sim p_{d, t, O}(\Sigma)$, where $Q = OP_d^\top P_d O^\top$ and  %
\[
p_{d, t, O}(\Sigma):= \cN(0, a_t^2Q+b_t^2 Q \Sigma Q)\, .
\]
In practical scenarios, the true covariance matrix $\Sigma$ is unknown and must be estimated from data. The derivation above only depends on $p_0$ through its covariance matrix $\Sigma$. Thus, by estimating the covariance matrix with $\hat\Sigma$, the estimated backward diffusion process follows a Gaussian distribution given by (see Appendix \ref{app:discussion} for a formal discussion)
\[
p_{d, t, O}(\hat\Sigma) = \cN(0, a_t^2 Q+b_t^2 Q\hat\Sigma Q)\,.
\]
Crucially, the expression \eqref{eq:add-noise} together with the pseudo-inverse $P^+$ reveal that the generative process can be reinterpreted as a ``noisy autoencoder'': a model that encodes the data, injects noise into the latent representation, and decodes it. While this equivalence between the learned noisy autoencoder and score-based diffusion model breaks when moving beyond Gaussian data (or restricting the class of score functions as in Section \ref{sec:erm}), the insights from this scenario still apply to more complex data, as we will illustrate numerically in Section \ref{sec:experiment}.

Our goal is to quantify the discrepancy between the target distribution $p_0$ and the decoded backward distribution $p_{d, t, O}(\Sigma)$ as a function of the latent dimension $d$ and time $t$. To this end, we utilize the Wasserstein-2 distance \citep{villani2008optimal}, which corresponds to the Fr\'echet distance in the Gaussian setting \citep{heusel2017gans}:
\begin{align}
\label{eq:frechet-distance}
d_F^2(\cN(\mu_1, \Sigma_1)&, \cN(\mu_2, \Sigma_2)) = \|\mu_1-\mu_2\|_2^2\nonumber\\
&\quad+\tr(\Sigma_1+\Sigma_2 - 2(\Sigma_2^{1/2}\Sigma_1\Sigma_2^{1/2})^{1/2}).
\end{align}
With a slight abuse of notation, we denote by $d_F(X, Y)$ the distance between the distributions of the random variables $X$ and $Y$. We select this metric since it is the \textit{de facto} standard to evaluate generative models; unlike the Kullback–Leibler (KL) divergence, which is ill-defined for distributions with degenerate support, the Fr\'echet distance remains valid in our projection-based setup. Furthermore, we expect related metrics such as Sliced-Wasserstein (SW) and Maximum Mean Discrepancy (MMD) distances to exhibit similar behavior---see Appendix~\ref{app:exp} for experimental results with these two metrics. For simplicity, we define our primary quantity of interest as:
\[
\Delta_{d, t, O}(\Sigma) := d_F^2(p_0, p_{d, T-t}(\Sigma)) \, .
\]
Whenever $O=I_D$ we simply denote it by 
$\Delta_{d, t}(\Sigma) := \Delta_{d, t, I_D}(\Sigma).$
We note that $\Delta_{d, t, O}(\Sigma)$ is defined for $t\in [0, T]$ and $d \in \{1, \dots, D\}$.
We will also be interested in the plug-in estimator $\Delta_{d, t}(\hat\Sigma) = d_F^2(p_0, p_{d, T-t}(\hat\Sigma))$.

\section{Optimal dimension reduction and stopping time}
\label{sec:main-result}
In this section, we address the important question of how dimensionality reduction affects the diffusion process with respect to the intrinsic geometric structure of the data. Our analysis focuses on selecting the rank $d$ of the projection matrix $P_d$ and the stopping time of the diffusions, by examining the dependence of $\Delta_{d, t, O}(\Sigma)$ and $\Delta_{d, t, O}(\hat \Sigma)$ on $d$ and~$t$. We begin by the simpler case where the covariance matrix is assumed to be diagonal before moving on to a general covariance matrix (with unknown eigenvectors).

\subsection{Non-monotonicity of the Fr\'echet distance}
In this subsection and the next one, we focus on the scenario where $p_0$ has independent components, i.e. $\Sigma = \textnormal{diag}(\sigma_1^2, \hdots, \sigma_D^2)$, and diagonal estimators $\hat\Sigma=\textnormal{diag}(\hat\sigma_1^2,\hdots,\hat\sigma_D^2)$, where the diagonal elements are the sample variances $\sigma_k^2=\frac1n\sum_{i=1}^nX_{ik}^2$. We assume that $n$ is sufficiently large such that $\hat\sigma_1^2\geq\hdots\geq\hat\sigma_D^2>0$ holds with high probability. In addition, we  consider $O=I_D$.

This subsection examines the non-monotonic behavior of the Fr\'echet distance as a function of diffusion timesteps, challenging the intuitive expectation of monotonic evolution. The common belief of monotonicity \citep{jayasumana2024rethinking} implies that a stopping time closer to $T$
consistently yields a smaller Fr\'echet distance.
First, we derive a necessary and sufficient condition for this non-monotonicity to occur in the scenario where the target distribution is Gaussian, as in \eqref{eq:independent-gaussian}. The proof of this result, as well as those of the subsequent ones, can be found in the Appendix.
\begin{proposition}
\label{prop:non-monotonic-distance}
For $d\in\{1,\hdots,D\}$, the Fr\'echet distance $\Delta_{d, t}(\Sigma)$ is non-increasing with respect to $t$. On the other hand, $\Delta_{d, t}(\hat\Sigma)$ is non-increasing if and only if 
\begin{equation}
\label{eq:condition-non-monotonic-distance}
\sum_{d'=1}^d (1- \frac{\sigma_{d'}}{\hat \sigma_{d'}})(1 - \hat\sigma_{d'}^2) \geq 0.
\end{equation}
\end{proposition}

Roughly speaking, the variance of each backward diffusion component $P_d\bwdhatX t$ scales monotonically from an initial value close to 1 ($a_T^2+b_T^2\hat{\sigma}_{d'}^2 \approx 1$) to the estimated variance $\hat{\sigma}_{d'}^2$ or any values satisfying~\eqref{eq:condition-non-monotonic-distance}, e.g., when $\hat\sigma_{d'}^2$ are given by an oracle. The distance $\Delta_{d, t}(\hat\Sigma)$ is therefore minimized when the process variance is closest to the true set of variances $(\sigma_{d'}^2)_{1 \leq d' \leq d}$, which happens before time $T$ under condition \eqref{eq:condition-non-monotonic-distance}. %

For a clearer understanding, consider the scenario where $p_0$ is a distribution lying in a linear subspace that is isomorphically equivalent to $\R^{d_0}$. In other words, suppose that $\sigma_{D}=\hdots=\sigma_{d_0+1}= 0$, and also $\hat\sigma_{D}=\hdots=\hat\sigma_{d_0+1}= 0$. Let us first consider the case where there is no projection, i.e., $d=D$. Then, the left-hand side of \eqref{eq:condition-non-monotonic-distance} can be rewritten %
\begin{align*}
\sum_{d'=1}^D (1-\frac{\sigma_{d'}}{\hat\sigma_{d'}})(1-\hat\sigma_{d'}^2) &=\sum_{d'=1}^{d_0} (1-\frac{\sigma_{d'}}{\hat\sigma_{d'}})(1-\hat\sigma_{d'}^2)\\
&\quad+D-d_0.
 \end{align*}
For large enough $n$ and with high probability, $|\sigma_{d'} - \hat\sigma_{d'}|\leq\hat\sigma_{d'}$ for every $d'\in\{1,\hdots, d_0\}$, we may deduce that
\begin{align*}
\sum_{d'=1}^D (1-\frac{\sigma_{d'}}{\hat\sigma_{d'}})(1-\hat\sigma_{d'}^2)\geq D- \big(2 + \max_{d'\in\{1,\hdots, d_0\}} \hat \sigma_{d'}^2\big)d_0.
\end{align*}
The last term is positive as long as the ambient dimension $D$ is large enough. Therefore, in this context, $d_F(\bwdhatX t, \orr X_0)$ is non-increasing.

However, if projecting the diffusion onto the $d_0$-dimensional linear subspace in which the data distribution lies, the $D-d_0$ term in the computation above vanishes, and we are left with the sum up to $d_0$. Then the behavior of the Fr\'echet distance is linked to how the model estimates the variances of the data. If, for most $d'$, the sign of $1-\sigma_{d'}/\hat\sigma_{d'}$ matches the sign of $1 - \hat\sigma_{d'}^2$, the Fr\'echet distance exhibits monotonic behavior. Conversely, if most of the signs differ, the Fr\'echet distance is non-monotonic. In addition, for a given estimation error, non-monotonicity is more likely to occur when the latent dimension is small.

This insight suggests that early stopping can improve the backward diffusion process, bringing the generated distribution closer to the data distribution. We next ask the reverse question: given a stopping time $t$, what is the optimal latent dimension?

\subsection{Optimal projection at time \texorpdfstring{$t$}{t}}

In this subsection, we continue the study of the interaction of the dimension of projection and the stopping time. 
\begin{wrapfigure}[13]{r}{0.4\linewidth}
\includegraphics[width=\linewidth]{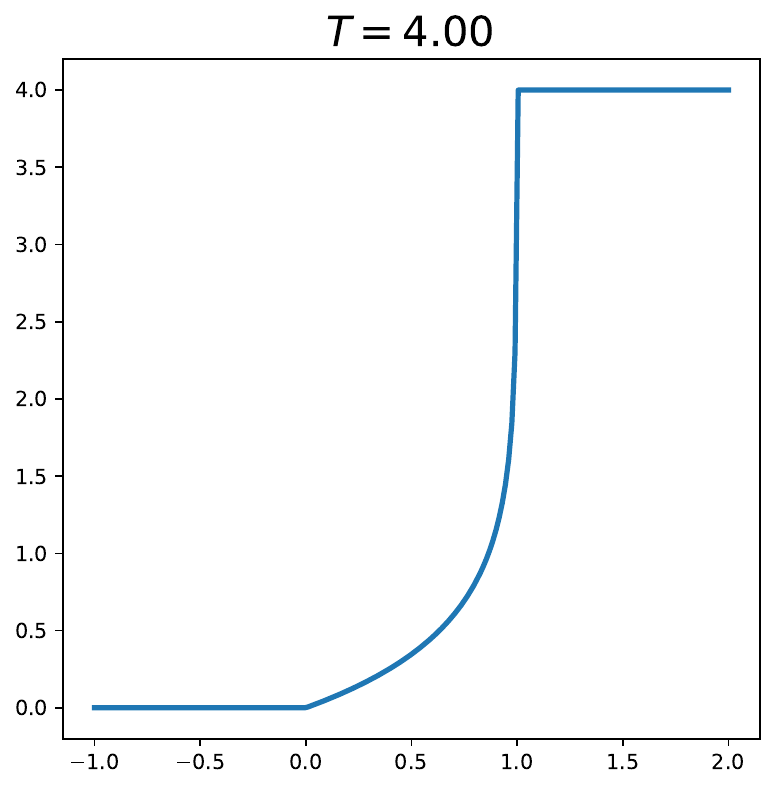}
    \captionof{figure}{$\bar{a}^{-2}$ in the Ornstein-Uhlenbeck process.}
    \label{fig:bar_a_minus2}
\end{wrapfigure}
In contrast to the previous subsection, we show that for each fixed time $t$, there exists an optimal projection $P_d$. %
We still consider Gaussian data with independent components 
\eqref{eq:independent-gaussian}. 
Recall that $a$ is defined in \eqref{eq:forward-distribution} and that it is an increasing map from $[0,T]$ to $[0, a_T]$.
We then let $\bar{a}^{-2}:\R\cup\{\infty\}\to[0, T]$ be the extended inverse function of $a^2$ (see plot in Figure \ref{fig:bar_a_minus2}), meaning that
\begin{equation}\label{eq:bar-a-minus2}
      \bar{a}^{-2}(x) = \begin{cases} 
        0, & \text{for } x < 0, \\ 
        a^{-2}(x), & \text{for } x\in[0, a_T^2],\\ 
        T, & \text{for } x \in (a_T^2, \infty].
      \end{cases}
\end{equation}
In particular, for $t\in[0, T]$, $\bar{a}^{-2}(a_t^2) = t$. For $d\in\{2,\hdots, D\}$, we then let
\[
t_d = T - \bar{a}^{-2}\Big(\frac{3 \sigma_d^2}{(1 - \sigma_d^2)_+}\Big)\]
and
\[\hat t_d = T - \bar{a}^{-2}\Big(\frac{4 \sigma_d^2-\hat\sigma_d^2}{(1 - \hat\sigma_d^2)_+}\Big).
\]
By convention, we let $\hat t_1 = t_1=0$ and $\hat t_{D+1} = t_{D+1}=T$. Observe that the times $t_d$ are in increasing order and between $0$ and $T$. 
Given these time partitions, we can characterize the optimal projection dimension, both for the exact backward process and  the one incorporating score estimation, with the aim of minimizing the distance between the generated and target distributions.
\begin{proposition}
\label{prop:min_wasserstein_projected} Assume that $0<\sigma_D< \cdots < \sigma_1$. Then, for $d\in\{1,\hdots, D\}$ and~$t\in[t_d, t_{d+1})$,
\[
d = \argmin_{d'\in\{1, \hdots, D\}}\Delta_{d', t}(\Sigma)
\]
Furthermore, with high probability, the $\hat \sigma_d$ and the $\hat t_d$ are well-ordered. In this case,
for $t\in[\hat t_{d}, \hat t_{d+1})$ 
\[
d = \argmin_{d'\in\{1,\hdots, D\}}\Delta_{d', t}(\hat\Sigma)
\]
\end{proposition}

\begin{figure}[ht]
    \hfill
    \includegraphics[width=.40\linewidth]{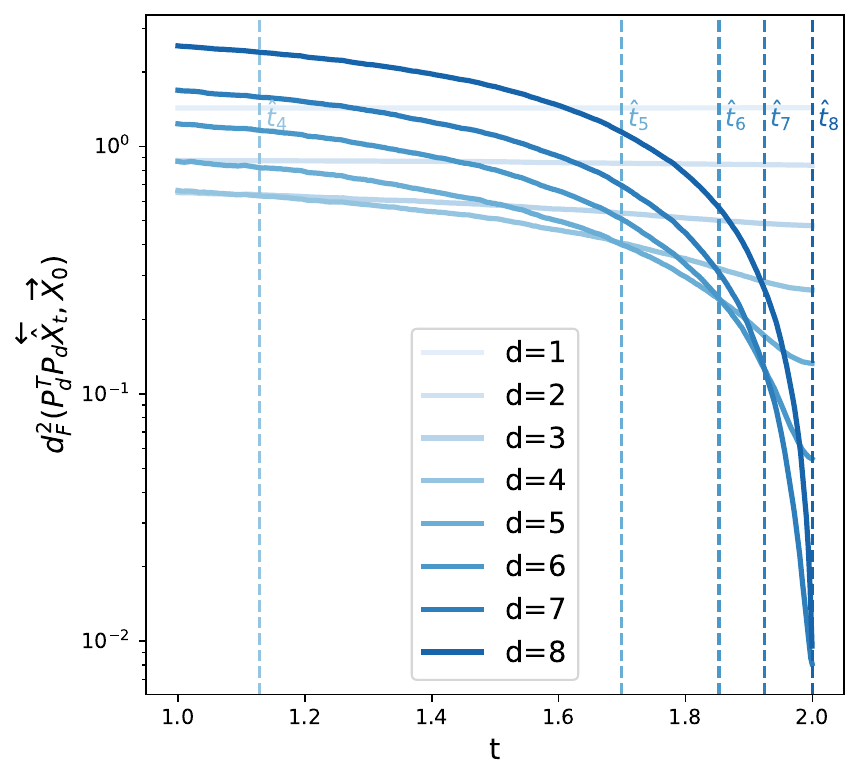}
    \hfill
    \includegraphics[width=.40\linewidth]{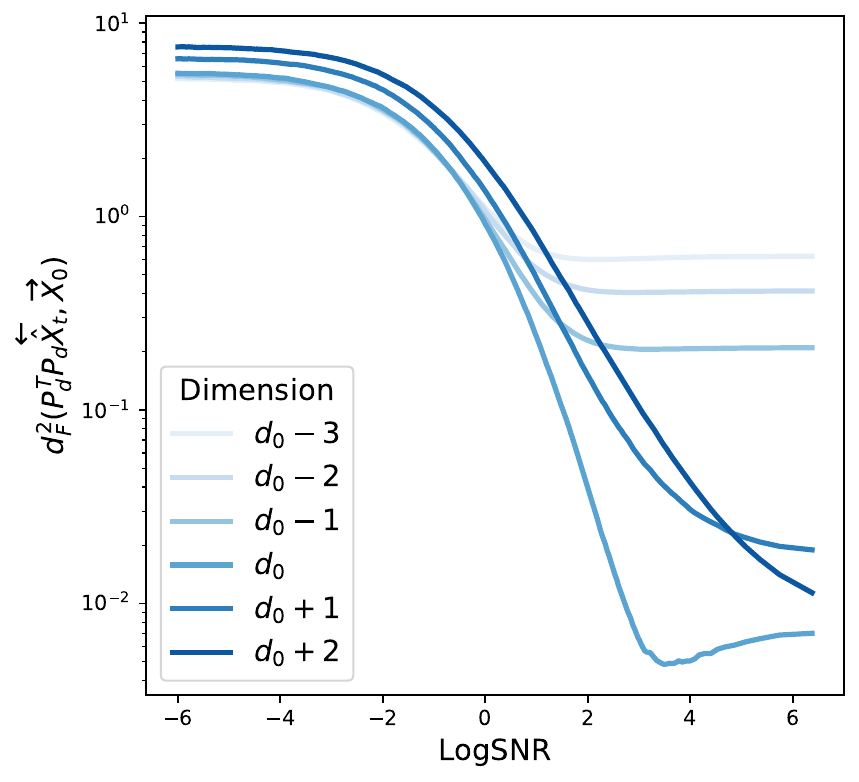}
    \hfill\hfill
    \caption{Plots of $\Delta_{d, t}(\hat\Sigma)$ as a function of the diffusion time $t$, for two sets of variances. (left) All the $\sigma_i$ are nonzero. As expected from Proposition \ref{prop:min_wasserstein_projected}, the $d$-dimensional projection is optimal in $[t_d, t_{d+1})$. (right) The data is supported on a linear subspace of dimension $d_0=4$ with $D=6$. As expected from Proposition \ref{cor:isotropic_case-optimal-stopping-time}, we observe that the minimum distance is achieved in dimension $d_0$ and with early stopping. LogSNR in the $x$-axis is a remapping of time $t$, defined as $\log(b_t^2/a_t^2)$, which we use to increase readability. Experimental details are in Appendix \ref{app:exp}.}
    \label{fig:projection_experiments}
\end{figure}

Proposition \ref{prop:min_wasserstein_projected} quantifies a direct link between early stopping in the backward diffusion process and dimensionality reduction. 
It reveals a time-dependent trade-off: distributions at early stages of the backward process are best approximated in lower-dimensional spaces, while higher dimensions become necessary to faithfully reconstruct the data as $t \to T$, as illustrated in Figure~\ref{fig:projection_experiments} (left). In other words, at an early time step, projecting onto an unnecessarily high-dimensional space can introduce more noise than signal, making a lower-dimensional representation more accurate. Notice that when $4\sigma_d^2\geq 1$, both $t_d$ and $\hat t_d$ are equal to $0$. This implies that a component whose variance is sufficiently large should always be included in the projection, aligning with the intuition that major components are essential for representation. These results hold for backward processes using scores based on either the true or empirical variances.

We next characterize the behavior of the optimal latent dimension and stopping time when data lies on a $d_0$-dimensional subspace, providing a similar result to Proposition \ref{prop:min_wasserstein_projected}. This analysis allows us to precisely determine these two key parameters, as shown next.

\begin{proposition}
\label{cor:isotropic_case-optimal-stopping-time}
Let $\Sigma=\textnormal{diag}(\sigma^2,\hdots,\sigma^2, 0, \hdots, 0)$ with the last $D-d_0$ entries equal to $0$. Let $\varepsilon \in (0,1)$. Then, there exists $\hat\delta_{d_0}\in[0, T]$ such that with probability $1 - 2d_0e^{-\frac{n}{8}}$,
\[
(d_0, T-\hat\delta_{d_0}) = \argmin_{(d', t) \in \{1, \hdots, D\} \times [0, T]}\Delta_{d', t}(\hat\Sigma).
\]
\end{proposition}
The proposition shows that the optimal generation strategy for data with a low-rank structure involves both early stopping and projection (see Figure~\ref{fig:projection_experiments} (right) for an illustration).
The proof indicates that, under the non-monotonicity condition of Proposition \ref{prop:non-monotonic-distance}, the optimal early stopping time $T- \hat\delta_{d_0}$ is strictly before $T$. Beyond preventing numerical instability as $t \to T$ \citep[e.g.,][]{yang2023lipschitz}, Proposition \ref{cor:isotropic_case-optimal-stopping-time} thus offers a new justification for early stopping. In other words, stopping at a specific time $\hat{\delta}_{d_0}$ is not merely a practical fix, but an optimal strategy to improve generation quality by minimizing the distance between the generated and true data distributions.%

Furthermore, this result confirms the intuition that confining the generative process to the dataset's intrinsic dimensionality is the most effective approach for low-rank data. This strategy is not only computationally more efficient than running the diffusion in the ambient space, but also enhances generation quality by avoiding the noise introduced by superfluous dimensions. The assumption of exact low-rank data is sometimes relaxed in the literature to an assumption on a decaying spectrum of eigenvalues. In the setting presented in this section where we plug-in the estimated covariance matrix into the true score, running the backward diffusion in the full data dimension is always optimal when the data has  decaying but nonzero eigenvalues. In Section \ref{sec:erm}, we move to a setting closer to practice, where the score is learned by minimizing a constrained empirical risk. We will then prove the existence of a smaller optimal dimension for data with a decaying spectrum.

\subsection{Generalization to arbitrary Gaussian distributions}
\label{sec:general-gaussian}
We now explain how to generalize some of our preceding analysis from Gaussian distributions with diagonal covariance matrices to the more general case $p_0 = \mathcal{N}(0, \Sigma)$ for arbitrary $\Sigma$, and the backward processes with the new general data distribution $p_0$.
Our goal is to establish a result analogous to Proposition~\ref{prop:min_wasserstein_projected}, that is, to characterize the optimal latent dimension given a stopping time of the diffusion process. 

To this end, let $\Sigma = O\Lambda O^\top$ be the eigen decomposition of $\Sigma$, where $O$ is an orthogonal matrix and $\Lambda$ is the diagonal matrix of eigenvalues, which we assume are distinct and ordered $\sigma_1^2 > \hdots > \sigma_D^2 > 0$. 
As in Section~\ref{sec:main-result}, we define a time partition by setting $t_1=0$ and $t_{D+1} = T$, and defining the intermediate timesteps for $d\in\{2, \hdots, D\}$ as:
\[
t_d = T-\bar{a}^{-2}\left(\frac{3\sigma_d^2}{(1-\sigma_d^2)_+}\right),
\]
where $\bar{a}^{-2}$ is given in \eqref{eq:bar-a-minus2}.
This definition, combined with the ordering of the eigenvalues, yields a sequence $0 = t_1 \leq t_2 \leq \cdots \leq t_D \leq t_{D+1} = T$. We show next that for this general Gaussian case, PCA projection onto $d$ components is optimal precisely within the interval $[t_d, t_{d+1})$.

\begin{proposition} 
\label{prop:general-independent-gaussian-min-wasserstein}
For $2\leq d\leq D$ and $t\in[t_d, t_{d+1})$, we have
\[
d = \argmin_{d'\in\{1,\hdots, D\}} \Delta_{d, t, O}(\Sigma)\,.
\]
\end{proposition}
However, in practical applications, one rarely has access to the true underlying covariance matrix $\Sigma$ or its eigenbasis $O$. Instead, one must rely on estimations derived from observed data, where PCA is commonly used. Denote $\hat\Sigma=\frac1n\sum_{i=1}^n X_iX_i^\top$ to be the empirical covariance matrix. Applying a spectral decomposition yields $\hat\Sigma=\hat O \hat \Lambda \hat O^\top$, where $\hat O$ contains the orthonormal eigenvectors and $\hat \Lambda = \textnormal{diag}(\hat\sigma_1^2,\hdots,\hat\sigma_D^2)$ where $\hat\sigma_D^2<\hdots<\hat\sigma_1^2$ are the corresponding eigenvalues. Denote $S(\Sigma) = \sum_{d'=1}^D\max(\sigma_d, \sigma_d^2)$. For $u\geq 0$ and $d\in\{2,\hdots, D\}$, we let $\hat T_d(u) = T - \tau^-_d(u)$ and $\hat t_d(u) = T-\tau^+_d(u)$ where
\begin{align*}
\tau^\pm_d(u) = \bar{a}^{-2}\Bigg(\frac{\hat\sigma_d^2 
\pm4S(\Sigma)\varepsilon_u + 2\hat\sigma_d\sqrt{\hat\sigma_d^2-4S(\Sigma)\varepsilon_u}}{(1 - \hat\sigma_d^2)_+}\Bigg),
\end{align*}
where $\varepsilon_u=\frac{8C}{3}(\sqrt{\frac{D+u}{n}}+\frac{D+u}{n})$. We assume that $\varepsilon_u$ is sufficiently small (i.e., $n$ large enough) so that the square root in the definition above is well-defined and the argument of $\bar{a}^{-2}$ is positive. By convention, we set $\hat T_1(u) = 0 $ and $\hat t_{D+1}(u) = T$. Thus, for small $\varepsilon_u$, these timesteps are ordered as 
\[
0 = \hat T_1(u)< \hat t_2(u)< \hat T_2(u)<\cdots<\hat t_{D+1}(u) = T.
\]
We are now in a position to describe the optimal projection strategy at each stopping time.
\begin{proposition}
\label{prop:general-gaussian-min-wasserstein}
For $d\in\{1,\hdots,D\}$ and any $t \in [\hat T_d(u), \hat t_{d+1}(u)]$, with probability $1-2e^{-u}$,
\[
d= \argmin_{d'\in\{1,\hdots,D\}}\Delta_{d, t, \hat O}(\hat \Sigma)\,.
\]
\end{proposition}
This proposition generalizes the result of Proposition \ref{prop:min_wasserstein_projected} to the case of a general Gaussian data distribution. The analysis reveals that, for any latent dimension $d$, there exists a time interval where a $d$-dimensional projected diffusion process minimizes the distance to the target distribution with high probability. Notably, this result is consistent with our previous conclusions. In the idealized scenario where the variance estimation error is zero (i.e., $\varepsilon_u=0$, implying $\hat\Sigma=\Sigma$,) the formula for the optimal time $\hat t_d=\hat T_d$ simplifies precisely to the one derived in Proposition \ref{prop:general-independent-gaussian-min-wasserstein}.

\section{Performance of the score matching ERM}
\label{sec:erm}

In the previous section, we analyzed the properties of diffusion processes with a score tailored to independent Gaussian distributions involving either exact or plugged-in estimated variances. In practice, the score is rather \textit{learned} by solving a regression problem called score matching. Specifically, given a training sample $(X_1, \hdots, X_n)$ independently drawn from the data distribution $p_0$, the empirical score matching objective writes
\begin{equation}    \label{eq:score-matching-obj}
\cR(s) = \frac{1}{n} \sum_{i=1}^n \E_{t \sim \mathcal{T}, \varepsilon \sim \mathcal{N}(0, I_D)}\left\|s(b_t X_i+a_t\varepsilon, t)+\frac{\varepsilon}{a_t}\right\|^2  \end{equation}
for some absolutely continuous distribution $\mathcal{T}$ with positive mass over $[0, T]$, and where the predictor $s: \R^D \times \R \to \R$ belongs to some hypothesis class $\mathcal{F}_{C}$, typically a neural network architecture.
In our context, recall that the score function of a Gaussian distribution with diagonal covariance $\Sigma$ is
\begin{equation}
\label{eq:true-score-function}
\nabla\log p_t(x) = -(a_t^2I_D+b_t^2\Sigma)^{-1}x,
\end{equation}
which takes the form of a time-dependent diagonal matrix multiplied by $x$. Thus a natural choice of hypothesis class given the form of the true score function \eqref{eq:true-score-function} is, for $C>1$,
\begin{align*}
\mathcal{F}_{C} =\left\{s_M: \mathbb{R}^d\right.&\times\mathbb{R}_+ \to \mathbb{R}^d: s_M(x, t) = -M(t)x, \qquad \\ &M(t)=\textnormal{diag}(m_1(t), \hdots, m_D(t))\\
&  \left. m_i \in \mathcal{L}_2(\mathbb{R}_+, \mathbb{R}), \|m_i\|_\infty<C\right\}.
\end{align*}
The assumption of $C>1$ is essential since we start the backward diffusion from a standard Gaussian distribution whose score function is the identity function. We introduce the norm constraint on the weights to account for two phenomena. First, the norm of the true score function \eqref{eq:true-score-function} blows up for times close to $0$ (in particular if the covariance matrix is singular or close to singular), which is known to create numerical instabilities \citep{ lu2023mathematical, yang2023lipschitz}. This is mitigated in practice for instance by early stopping the diffusion. Here, we implement this mitigation by capping the weight norm. Second, it is known that gradient descent has an implicit bias towards learning low-norm solutions. Although quantifying this effect is beyond the scope of this paper, the explicit weight constraint provides an analytically tractable analogue.
More precisely, one can easily derive the following explicit formula for the minimizer of the score matching over $\mathcal{F}_{C}$.%
\begin{proposition} \label{prop:optimal-constrained-score}
Let $\hat\sigma_d^2 = \frac{1}{n}\sum_{i=1}^n X_{id}^2$ be the empirical variance for the $d$-th component of the training data.  Then the minimizer of the score matching objective \eqref{eq:score-matching-obj} over $\mathcal{F}_{C}$ is given by $\hat M(t) = \textnormal{diag}(\hat m_1(t), \hdots, \hat m_D(t))$ where, for $d\in\{1,\hdots , D\}$,
\[
\hat m_d(t) = \min\left(C, \frac{1}{a_t^2+b_t^2\hat\sigma_d^2}\right) . %
\]
\end{proposition}
Our goal in the following is to characterize the optimal latent dimension when using the score defined by Proposition \ref{prop:optimal-constrained-score}. For this purpose, as before, we quantify the distance between the data distribution and the distribution generated by the backward process for $d\in\{1,\hdots, D\}$. While we do not consider early stopping here, to focus on the influence of the regularization parameter~$C$ on the choice of the latent dimensionality, we note that equation \eqref{eq:estimated-backward-diffusion-distribution} in Appendix \ref{app:discussion} gives the formulas for the distribution of the backward diffusion at any time $t>0$, from which the Fréchet distance for the early-stopped process could be derived and numerically studied. For simplicity, we keep the data distribution $p_0$ to be a Gaussian distribution with independent component, and specialize to the Ornstein-Uhlenbeck process. In this case, the sample $\bwdtildeX t$ are generated by the backward SDE for $t \in [0, T]$
\begin{align*}
d\bwdtildeX t &= (\bwdtildeX t +2s_{\hat M}(\bwdtildeX t, T-t))dt + \sqrt2d\orl{W_t},\\
\bwdtildeX 0&\sim\cN(0, I_D).
\end{align*}
Note that we consider the standard setting in which the backward process starts from a standard Gaussian. We can then characterize the optimal projection for the latent diffusion, as shown next.
\begin{proposition} \label{prop:learned_score_optimality}
Define $1\leq d_1 \leq d_2 \leq D$ as follows:
\[
d_1=\max\{d'\in\{1,\hdots,D\}: 1/C\leq\hat\sigma_{d'}^2\}\]
and
\[d_2=\min\Big\{d'\in\{1,\hdots, D\}:\frac{1}{2C-1} \geq 4\sigma_{d'}^2\Big\}.
\]
(If the corresponding set in their definition is empty, we let $d_1=1$ and $d_2=D$, respectively.)
Then, with high probability, there exists an optimal projection dimension $d_1\leq d_{\min}\leq d_2$ such that
\[
d_{\min}= \argmin_{d'\in\{1,\hdots,D\}}d_F(P_{d'}^\top P_{d'}\bwdtildeX T, \fwdX 0)\,.
\]
\end{proposition}
To gain intuition into Proposition \ref{prop:learned_score_optimality}, let us consider some illustrative cases depending on the weight constraint $C$ (for the full derivation of these cases, see Appendix \ref{app:derivation-special-case-erm}). First, when $C=\infty$ and the data covariance is non-singular, we get $d_1=d_2=D$, thus Proposition \ref{prop:learned_score_optimality} suggests to take the projection matrix $I_D$, which is expected since $\mathcal{F}_C$ is then large enough to contain the true empirical score function. Second, consider the scenario when the data distribution lies on a linear subspace of dimension $d_0$. If $C$ is large enough, we obtain $d_1=d_2=d_0$, meaning that the projection onto the data subspace is the optimal sampling strategy, which is in line with Proposition~\ref{cor:isotropic_case-optimal-stopping-time}. 
Finally, the optimal projection can also be made explicit for exponentially-decaying covariance spectrum. 
\begin{corollary}
\label{cor:special-case}
Let $\lambda>16$. Assume that $\Sigma=\textnormal{diag}(\lambda^{-1}, \hdots, \lambda^{-D})$ and $\lambda\leq C\leq\lambda^D$. Let $d\in\{1,\hdots,D\}$ be such that $\hat\sigma_{d+1}^2\leq 1/C \leq \hat\sigma_d^2$. Then, with $n$ large enough and high probability, 
\[
d_{\min}\in\{d, d+1\}.
\]
\end{corollary}

Interestingly, when the covariance structure decays exponentially, the capacity of the score-function class—captured here by the parameter $C$—is directly tied to the optimal projection dimension. The latter scales logarithmically in the parameter $C$ since $\hat \sigma_{d_{\min}}^2 = \lambda^{-d_{\min}} \approx 1/C$. The result is illustrated in Appendix \ref{app:exp} Figure \ref{fig:linear-model-score-matching}.
An analogous result can be derived for covariance structures with power-law decay.

\section{Empirical analysis}
\label{sec:experiment}
In this section, we evaluate our theoretical framework on complex natural images with non-linear architectures, moving beyond the Gaussian setting of Section \ref{sec:main-result} and \ref{sec:erm}. To assess whether the predicted behaviors persist, we investigate a "noisy AE" proxy model defined as follows. For a given sample $x\in\R^D$ and a pair of encoder and decoder $E:\R^D\to\R^d$ and $D:\R^d\to\R^D$, the noisy AE computes $D(b_tE(x)+a_tZ)$ where $Z\sim\cN(0, I_d)$, thereby simulating the diffusion corruption process at time $t$ within the latent space before decoding back to the pixel space.

We verify the robustness of our predictions on ImageNet-256 \citep{deng2009imagenet}. Our setup employs U-Net-based AEs across various latent resolutions ($32^2 \times 4$, $64^2 \times 3$, $32^2 \times 16$) and corresponding LDMs, controlling for model size by fixing the combined parameter count at 500M. We observe that the generation quality of the noisy AE and the LDM aligns closely throughout the sampling phase even when using real-world data and realistic architectures.

Figure \ref{fig:LDM-AE-imagenet} presents the FID trajectories for LDMs and their noisy AE counterparts. First, we observe a strict alignment in performance rankings: the intersection points of the FID curves across different dimensions are identical for both models. Consistent with our theory, larger latent dimensions require later stopping times. Second, we challenge the assumption that full denoising yields the highest qualit. We find that optimal quality is not achieved at full denoising ($t=T$) but at an earlier time ($t<T$) (Figure \ref{fig:LDM-AE-imagenet}, top), likely due to high-frequency decoder artifacts \citep{odena2016deconvolution}. Crucially, this optimal stopping time for the LDM coincides with the time minimizing the noisy AE's FID (Figure \ref{fig:LDM_AE_FID_min_match}). Similar findings for MNIST \citep{deng2012mnist} and CelebA-HQ \citep{liu2015faceattributes} are detailed in Appendix \ref{app:exp}.

Overall, our results confirm key theoretical predictions: the optimal latent dimensionality depends on diffusion time, and FID curves exhibit a U-shape. The strong alignment between LDMs and noisy AEs suggests that the latter can serve as an efficient proxy for model selection, potentially allowing one to identify optimal hyperparameters—such as stopping time and latent dimensionality—without the computational cost of training full LDMs. Exploring this
hypothesis offers a path to circumvent the computational
bottleneck of training full LDMs for model selection.

While we focus on reporting the FID here since it is the most standard metric to assess the quality of image generative models, we also measure the generation quality of several LDMs on ImageNet-256 with MMD and Sliced Wasserstein distance (see Appendix \ref{app:exp}). We also find a U-shape in generation quality as a function of the diffusion time. Visual inspection further confirms that LDM-generated samples evolve only marginally during the final stages of diffusion, whereas pixel-space diffusion undergo much more pronounced refinement.

\begin{figure}[h!]
    \centering
    \includegraphics[width=.8\linewidth]{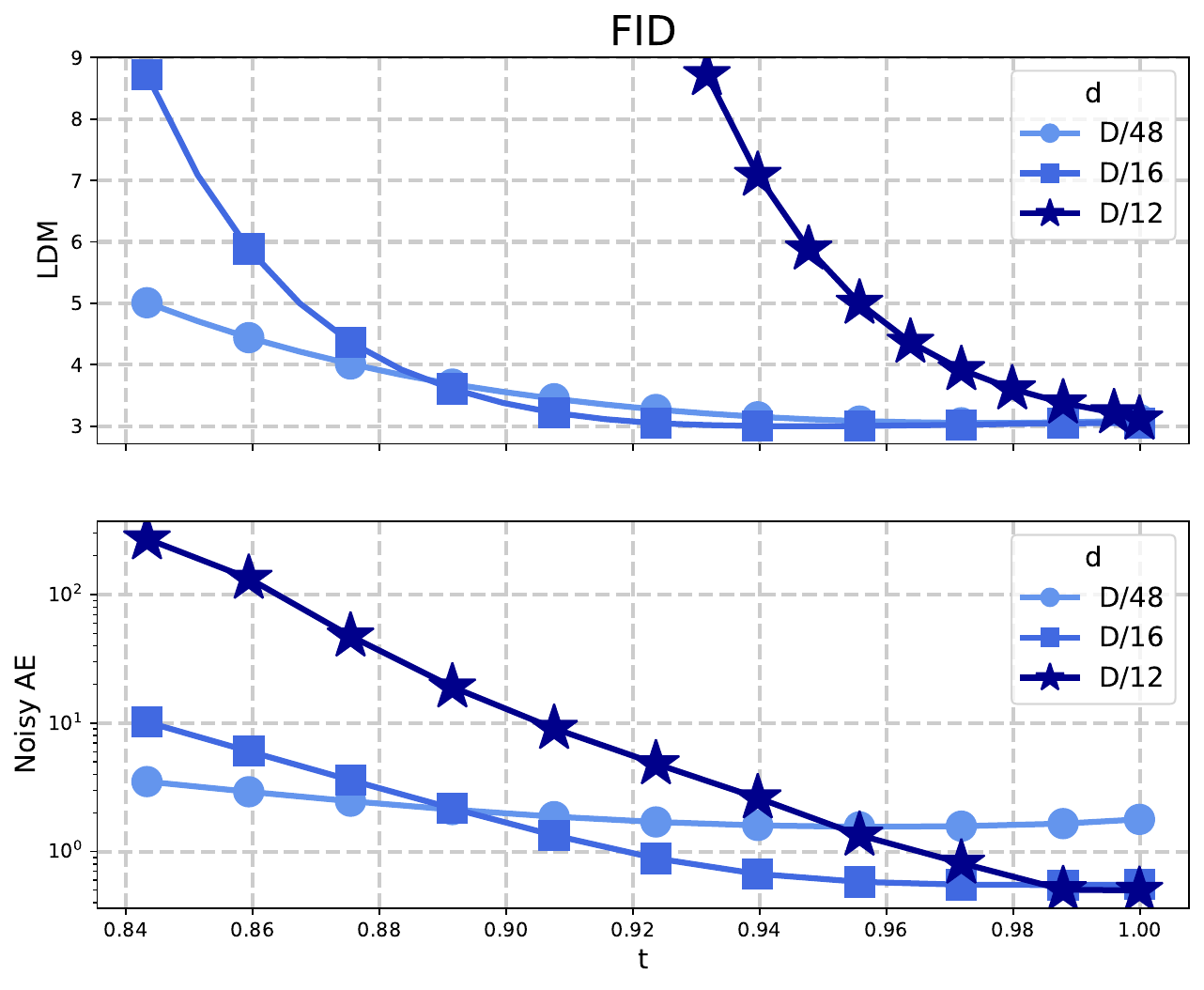}
    \caption{We train AEs and LDMs with latent dimensions $32\times32\times4$, $64\times64\times3$, and $32\times32\times16$ on ImageNet-256. The legend indicates the corresponding dimensionality reduction ratios. We report the FID scores here, and additionally evaluate performance using the sliced Wasserstein distance, with results provided in Appendix~\ref{app:exp}.
    }
    \label{fig:LDM-AE-imagenet}
\end{figure}
\begin{figure}[h!]
    \centering
    \includegraphics[width=.8\linewidth]{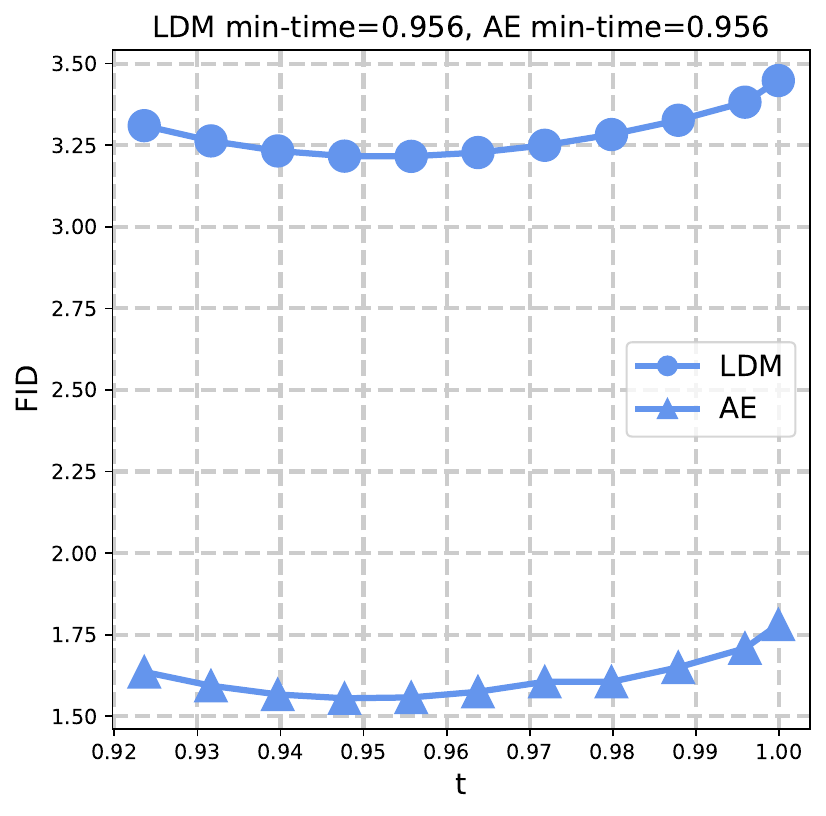}
    \caption{Zoom in of Figure \ref{fig:LDM-AE-imagenet} for the AE of dimension $32\times32\times4$.}
    \label{fig:LDM_AE_FID_min_match}
\end{figure}
\section{Conclusion}
This paper provides a theoretical analysis of optimal stopping time in latent diffusion models, showing its critical dependence on latent space dimensionality and its interaction with other hyperparameters of the diffusion process, such as weight regularization in the score matching phase. Our results focus on Gaussian distributions, given their tractability and prominence in prior theoretical works \citep{pierret2024diffusion, hurault2025score}. Taken together, these insights open compelling research directions, for deepening the theoretical properties of latent diffusion models and assessing when they can match or surpass the sampling quality of standard diffusion models.

\newpage
\section*{Acknowledgments}

P.M.~is supported by a Google PhD Fellowship.

\bibliography{ref}

\newpage
\onecolumn
\begin{center}
    \LARGE \bf Appendix
\end{center}
\appendix
\section{Discussion on diffusion and autoencoders}
\label{app:discussion}
Assume that the data distribution $p_0$ is centered Gaussian $cN(0, \Sigma)$. Then $P\bwdX {T-t}$ and $P\fwdX t$ also follow Gaussian distributions
\begin{equation}
\label{eq:forward-distribution}
P\bwdX {T-t} \overset{\cD}{=} P\fwdX t \sim \cN(0, P_d(a_t^2I_D+b_t^2\Sigma)P_d^\top)=\cN(0, a_t^2I_d+b_t^2P_d\Sigma P_d^\top),
\end{equation}
where the covariance matrix of $P_d\bwdX{T-t}$ effectively zeros out the last $D-d$ dimensions and  $a_t = \sqrt{1 - b_t^2}$ while $b_t =e^{-\int_0^t w_t^2 dt}$. A typical choice is the Ornstein-Uhlenbeck process, where $w_t\equiv1$, which implies $a_t = \sqrt{1-e^{-2t}}$ and $b_t=e^{-t}$.

When $p_0$ is a Gaussian distribution, the score function $\nabla\log p_t$ is completely determined by the covariance matrix $\Sigma$. Indeed, in \eqref{eq:forward-distribution} we have
\[
s_{P_d}(x, t) = -(a_t^2I_d+b_t^2P_d\Sigma P_d^\top)^{-1}x,\quad x\in\R^d.
\] 
Therefore, we may simplify the learning of the score function to the task of covariance matrix estimation. Since we assumed that the components of $p_0$ are independent, we consider a class of estimators consisting of diagonal matrices $\hat\Sigma = \textnormal{diag}(\hat\sigma_1^2, \hdots, \hat\sigma_D^2)$, where its diagonal elements are the estimated variances given by $\frac1n\sum_{i=1}^nX_{id}^2$ for all $d\in\{1,\hdots, D\}$ where $X_{i}=(X_{i1},\hdots,X_{iD})\in~\R^D$. Furthermore, we assume $\hat\sigma_1\geq\hdots\geq\hat\sigma_D>0$, which is satisfied with high probability when $n$ is large enough (up to permutations of the indices in case of equality of some of the variances). %

Plugging the estimated covariance matrix in the final distribution of $P_d\fwdX T$ and in the score function $s_{P_d}$, we can define the estimated sampling procedure
\begin{equation}
\label{eq:SDE-backward-estimated}
dP\bwdhatX t = (w_{T-t}^2P\bwdhatX t + 2w_{T-t}^2\hat s_{P}(P\bwdhatX t, T-t))dt + \sqrt{2w_{T-t}^2}dP\orl{W_t},\quad P\bwdhatX 0\sim (P)_\#\hat p_T,
\end{equation}
where 
\[
\hat s_{P}(x, t) = -(a_t^2I_d+b_t^2P\hat \Sigma P^\top)^{-1}x \quad \textnormal{and} \quad  (P)_\#\hat p_T \sim \cN(0, a_T^2I_d+b_T^2P\hat \Sigma P^\top).
\]
By the same derivations as above, we have the following identity
\begin{equation}
\label{eq:estimated-backward-diffusion-distribution}
P\bwdhatX t\sim\cN(0, a_{T-t}^2I_d+b_{T-t}^2P\hat\Sigma P^\top).   
\end{equation}
Therefore, our primary quantity of interest $\Delta_{d, t, O}(\Sigma)$ (resp. $\Delta_{d, t}(\hat\Sigma)$) corresponds to the distance of the distribution of $OP_d\top P_dO^\top\bwdX t$ (resp. $OP_d^\top P_dO^\top\bwdhatX t$) to the true distribution $p_0$.

\section{Proofs of results}
\subsection{Proof of Proposition \ref{prop:non-monotonic-distance}}
We show the equivalent statement: $t\in[0, T]\mapsto \Delta_{d, t}(\hat\Sigma)$ is non-decreasing if and only if \eqref{eq:condition-non-monotonic-distance} holds. We start by calculating the Fr\'echet distance $\Delta_{d, t}(\hat\Sigma)$ by using \eqref{eq:frechet-distance}:
\begin{align*}
\Delta_{d, t}(\hat\Sigma) &= \sum_{d'=d+1}^D \sigma_{d'}^2 +  \sum_{d'=1}^d\Big(b_t^2\hat\sigma_{d'}^2+a_t^2 + \sigma_{d'}^2 - 2 \sigma_{d'}\sqrt{a_t^2+b_t^2\hat\sigma_{d'}^2}\Big)\\
&= \sum_{d'=d+1}^D \sigma_{d'}^2 + \sum_{d'=1}^d\Big(\sqrt{a_{t}^2+(1-a_t^2)\hat\sigma_{d'}^2} - \sigma_{d'}\Big)^2.
\end{align*}
Since $t\mapsto a_t^2$ is strictly increasing with $a_0=0$, the monotonicity of $\Delta_{d, t}(\hat\Sigma)$ with respect to $t$ is equivalent to the monotonicity with respect to $a_t^2$. By considering the function $f:[0,a_T^2]\to\R$ defined by
\[
f(x) = \sum_{d'=1}^d \Big(\sqrt{x + (1-x)\hat\sigma_{d'}^2}-\sigma_{d'}\Big)^2,
\]
we see that $\Delta_{d, t}(\hat\Sigma)$ is non-decreasing if and only if $f$ is non-decreasing. Additionally,
\[
f'(x) = \sum_{d'=1}^d(\sqrt{x + (1-x)\hat \sigma_{d'}^2} - \sigma_{d'})\frac{1 - \hat\sigma_j^2}{\sqrt{x + (1 - x) \hat\sigma_{d'}^2}} = \sum_{d'=1}^d\Big( 1- \frac{\sigma_{d'}}{\sqrt{x + (1-x)\hat \sigma_{d'}^2}}\Big)(1 - \hat\sigma_{d'}^2),
\]
and
\[
f''(x) = \sum_{d'=1}^d\frac{\sigma_{d'}(1-\hat\sigma_{d'}^2)^2}{2(x + (1-x)\hat\sigma_{d'}^2)^{3/2}}>0.
\]
Hence, $f$ is convex so it is non-decreasing if and only if $f'(0)\geq 0$. Therefore,
\[\Delta_{d, t}(\hat\Sigma)=f(a_t^2) + \sum_{d'=d+1}^D\sigma_{d'}^2\]
is non-decreasing if and only if $f'(0) \geq 0$, i.e., if and only if $\sum_{d'=1}^d(1-\frac{\sigma_{d'}}{\hat\sigma_{d'}})(1-\hat\sigma_{d'}^2) \geq 0$. This shows the second statement of the proposition. The monotonicity of $\Delta_{d, t}(\Sigma)$ can be shown by replacing $\hat \sigma_{d'}$ with $\sigma_{d'}$ in the derivative $f'$, which is $0$ when $a_t=0$.
\subsection{Proof of Proposition \ref{prop:min_wasserstein_projected}}
The first part of Proposition \ref{prop:min_wasserstein_projected} concerns the minimization of $\Delta_{d, t}(\Sigma)$. 
Recall that $t_d = T - \bar{a}^{-2}\Big(\frac{3\sigma_d^2}{1-\sigma_d^2}\Big)$.
To prove that $\Delta_{d, t}(\Sigma)$ achieves the minimum for $t \in [t_d, t_{d+1})$ (where the time interval is fixed), we will demonstrate how the distance $\Delta_{d, t}(\Sigma)$ behaves as a function of the projection dimension $d$. Specifically, we aim to show that, for any $d \in \{2, \hdots, D\}$,
\begin{align}
\label{aux:frechet_tk1}
\Delta_{d, t}(\Sigma)&\leq \Delta_{d-1, t}(\Sigma) \quad \textnormal{iff } t \geq t_d.
\end{align}
This inequality in turn implies that for a given $t$ in a fixed interval $[t_d, t_{d+1})$, the minimum distance $\Delta_{d, t}(\Sigma) = d_F^2(P_d^\top P_d\fwdX t, \fwdX0)$ is attained by the projected process $P_d\fwdX t$ in dimension $d$.

To establish them, we first explicitly compute the Fréchet distance $\Delta_{d, t}(\Sigma)$. Recall that the Fréchet distance between two zero-mean Gaussian distributions $\mathcal{N}(0, \Sigma_1)$ and $\mathcal{N}(0, \Sigma_2)$ is given by $\textnormal{Tr}(\Sigma_1 + \Sigma_2 - 2(\Sigma_1^{1/2}\Sigma_2\Sigma_1^{1/2})^{1/2})$, and that the covariance matrix of $P_d\bwdX t$ is equal to $P_d(a_{T-t}^2I_d + b_{T-t}^2\Sigma)P_d$.
Therefore, it is possible to calculate the Fr\'echet distance to the target for the projected processes directly, as, for any $d\in \{ 1, \hdots , D\}$,
\[
\Delta_{d, t}(\Sigma) =\sum_{j=1}^D\sigma_j^2+ \sum_{j=1}^d(a_{T-t}^2+b_{T-t}^2\sigma_j^2) -2\sum_{j=1}^d \sigma_j\sqrt{a_{T-t}^2+b_{T-t}^2\sigma_j^2},
\]
so that
\begin{align}
\delta_{d,t} &:= d_F^2(P_d^\top P_d\bwdX t, \orr X_0)-d_F^2(P_{d-1}^\top P_{d-1}\bwdX t, \orr X_0) \nonumber \\
 &= b_{T-t}^2\sigma_d^2+a_{T-t}^2 - 2 \sigma_d\sqrt{a_{T-t}^2+b_{T-t}^2\sigma_d^2}, \nonumber\\
 &=\sqrt{b_{T-t}^2\sigma_d^2+a_{T-t}^2}\Big(\sqrt{b_{T-t}^2\sigma_d^2+a_{T-t}^2} - 2\sigma_d\Big), \nonumber\\
 &=\sqrt{(1-a_{T-t}^2)\sigma_d^2+a_{T-t}^2}\Big(\sqrt{(1 - a_{T-t}^2)\sigma_d^2+a_{T-t}^2} - 2\sigma_d\Big), \nonumber\\
 &=\sqrt{\sigma_d^2+a_{T-t}^2(1-\sigma_d^2)}\Big(\sqrt{a_{T-t}^2(1-\sigma_d^2)+\sigma_d^2} - 2\sigma_d\Big).\nonumber
\end{align}
We see that $\delta_{d,t}$ has the same sign as the term in the parenthesis on the last line, which itself has the same sign as $a_{T-t}^2(1-\sigma_d^2) - 3\sigma_d^2$. Then,
\begin{itemize}
    \item if $\sigma_d\geq1$ or $\frac{3\sigma_d^2}{1-\sigma_d^2} \geq a_T^2$, $\delta_{d,t}$ is non-positive for all $t \in [0, T]$, while $t_d = 0$ by definition;
    \item otherwise, $\textnormal{diff}_{d,t}$ is non-positive if and only if $a_{T-t}^2 \leq \frac{3\sigma_d^2}{1-\sigma_d^2}$ which is equivalent to
    $$T-t \leq a^{-2}\Big(\frac{3\sigma_d^2}{1-\sigma_d^2}\Big)=T-t_d.$$
\end{itemize}
Putting things together, we obtain that $\Delta_{d, t}(\Sigma)-d_F^2(P_{d-1}^\top P_{d-1}\bwdX t, \orr X_0)$ is non-positive iff $t \geq t_d$, which is exactly \eqref{aux:frechet_tk1}.

The proof in the case of estimated variances can be derived in a similar fashion as long as the estimated variances $\hat \sigma_i$ and times $\hat t_i$ are well-ordered, which happens with high probability for a sufficiently large sample.
\subsection{Proof of Proposition \ref{cor:isotropic_case-optimal-stopping-time}}
We first state the full proposition.
\begin{proposition} 
\label{cor:app-sotropic_case-optimal-stopping-time}
Assume that $\Sigma=\textnormal{diag}(\sigma^2, \hdots, \sigma^2, 0, \hdots, 0)$ with the last $D-d_0$ entries equal to $0$, and the estimated variances are ordered as $\hat\sigma_1^2\geq\hat\sigma_2^2\geq\hdots\geq\hat\sigma_{d_0}^2$.
Let $\varepsilon \in (0,1)$. For
\[
t \in  \bigg[T-\bar{a}^{-2}\Big(\frac{3-\varepsilon}{1+\varepsilon}\frac{\hat\sigma_1^2}{1-\hat\sigma_1^2}\Big), T\bigg),
\]
with probability $1-2 d_0 e^{-\frac{\varepsilon^2n}{8}}$, we have 
\[
d_0 = \argmin_{d'\in\{1,\hdots, D\}}\Delta_{d, t}(\hat\Sigma).
\]
If, in addition,
\begin{equation}
\sum_{d'=1}^{d_0}(1-\frac{\sigma}{\hat\sigma_{d'}})(1 - \hat\sigma_{d'}^2) < 0,~\label{eq:condition-cor-1}
\end{equation}
then 
\[
\sum_{d'=1}^{d_0}(1-\frac{\sigma}{\sqrt{\hat\sigma_{d'}^2 + (1-\hat\sigma_{d'}^2)a_{t}^2}})(1 - \hat\sigma_{d'}^2) = 0,
\]
has a unique solution which we denote by $\hat\delta_{d_0}$. By convention, if the condition \eqref{eq:condition-cor-1} is not satisfied, we set $\hat\delta_{d_0}=0$. Then, with probability $1-2d_0e^{-\frac{\varepsilon^2n}{8}}$,
\[
(d_0, T-\hat\delta_{d_0}) = \argmin_{\substack{t\in[0, T]\\ d' \in \{1, \hdots, D\}}}\Delta_{d', t}(\hat\Sigma).
\]
\end{proposition}
Let $\varepsilon\in(0, 1)$. We first note that according to Proposition \ref{prop:estimation-error-independent-gaussian}, by the union bound, with probability $1-2d_0e^{-\frac{\epsilon^2n}{4(1+\epsilon)}}\geq1-2d_0e^{-\frac{\epsilon^2n}{8}}$ we have $|\sigma^2-\hat\sigma_d^2|\leq \varepsilon\sigma^2$ for all $d\in\{1,\hdots,d_0\}$.
We work under this event in the remainder of the proof.
In particular, for all $d\in\{1,\hdots,d_0\}$, $\sigma_d^2 = \sigma^2 \geq \hat\sigma_1^2/(1+\varepsilon)$. Thus, by separating cases depending on whether $4\sigma_d^2\leq1$, a short calculation gives that
\[
\min\Big(1, \frac{\frac{4}{1+\varepsilon}\hat\sigma_1^2-\hat\sigma_1^2}{1-\hat\sigma_{1}^2}\Big)\leq \min\Big(1, \frac{4\sigma_d^2-\hat\sigma_1^2}{1-\hat\sigma_{1}^2}\Big)\leq \frac{4\sigma_d^2 - \hat\sigma_d^2}{1-\hat\sigma_d^2}.
\]
The last inequality is derived as follows: if $4\sigma_d^2\leq 1$, then we use the fact that $x\mapsto\frac{a-x}{1-x}$ is non-increasing if $a<1$. On the other hand, if $4\sigma_d^2\geq1$, then $\frac{4\sigma_d^2-\hat\sigma_d^2}{1-\hat\sigma_d^2}\geq1$.
Hence, by the monotonic increase of $\bar{a}^{-2}$, 
\begin{align*}
\hat t_d &= T - \bar{a}^{-2}\Big(\frac{4\sigma_d^2 - \hat\sigma_d^2}{1-\hat\sigma_d^2}\Big)  \\
&= T - \bar{a}^{-2}\Big(\min\Big(1, \frac{\frac{4}{1+\varepsilon}\hat\sigma_1^2-\hat\sigma_1^2}{1-\hat\sigma_{1}^2}\Big)\Big)  \\
&\geq T - \min\Big(T, \bar{a}^{-2}\Big(\frac{3-\varepsilon}{1+\varepsilon}\frac{\hat\sigma_1^2}{1-\hat\sigma_1^2}\Big)\Big) \\
&= \max\Big(0, T - \bar{a}^{-2}\Big(\frac{3-\varepsilon}{1+\varepsilon}\frac{\hat\sigma_1^2}{1-\hat\sigma_1^2}\Big)\Big) \\
&= T - \bar{a}^{-2}\Big(\frac{3-\varepsilon}{1+\varepsilon}\frac{\hat\sigma_1^2}{1-\hat\sigma_1^2}\Big) .
\end{align*}
Thus, $t\geq T - \bar{a}^{-2}\Big(\frac{3-\varepsilon}{1+\varepsilon}\frac{\hat\sigma_1^2}{1-\hat\sigma_1^2}\Big)$ implies $t\geq\hat t_d$ for every $d\in\{1,\hdots, d_0\}$. On the other hand, $t < T = \hat t_d$ for all $d\in\{d_0+1, \hdots, D\}$ since $\sigma_d = \hat\sigma_d = 0$. From here we deduce the desired result applying Proposition \ref{prop:min_wasserstein_projected}.

In this second part, we study under the event where $|\sigma^2-\hat\sigma_d^2|\leq \sigma^2$ for every $d\in\{1,\hdots, d_0\}$, which holds with probability $1-2d_0e^{-n/8}$ by Proposition \ref{prop:estimation-error-independent-gaussian}. To prove the desired result, we first show that the minimum of the distance $\Delta_{d_0, t}(\hat\Sigma)$ is attained at $t = T-\hat\delta_{d_0}$, as per its definition. We consider two cases depending on whether condition~\eqref{eq:condition-cor-1} is satisfied. First, if condition~\eqref{eq:condition-cor-1} holds, the proof of Proposition~\ref{prop:non-monotonic-distance} establishes that $a_{T-\hat\delta_{d_0}}^2$ is the unique zero of the derivative $\frac{d}{da_t^2}\Delta_{d_0, t}(\hat\Sigma)$. This confirms that $T-\hat\delta_{d_0}$ is the unique minimizer of the distance. Conversely, if condition~\eqref{eq:condition-cor-1} is not satisfied, then $\hat\delta_{d_0}=0$. In this scenario, the squared distance $\Delta_{d_0, t}(\hat\Sigma)$ is a non-increasing function of $t$ and thus attains its minimum at the endpoint $t=T$. This result is consistent, as $t = T = T - \hat\delta_{d_0}$.

We remark by Proposition \ref{prop:min_wasserstein_projected} that, since $\hat t_d=T$ for every $d\in\{d_0+1, \hdots,D\}$, for every $t\in[0, T]$,
\[
\Delta_{d_0, T-\hat\delta_{d_0}}(\hat\Sigma)\leq \Delta_{d_0, t}(\hat\Sigma)\leq \Delta_{d, t}(\hat\Sigma).
 \]
Observe that $\hat t_1 = \max_{d\in\{1,\dots, d_0\}}\hat t_d$, which is in the same order of $\hat\sigma_d$. This is due to the fact that $x\mapsto\frac{a-x}{1-x}$ is non-increasing if $a<1$. Then from the proof of Proposition \ref{prop:min_wasserstein_projected} we deduce that, for $t\geq \hat t_1$ and $d\in\{1,\hdots,d_0\}$, that
\begin{equation*}
\Delta_{d_0, T-\hat\delta_{d_0}}(\hat\Sigma)\leq \Delta_{d_0, t}(\hat\Sigma) \leq \Delta_{d, t}(\hat\Sigma).
\end{equation*}
If $\hat t_1 = 0$, the proof is finished. Note that this is the case if $\sigma^2\geq1/4$, since $\frac{4\sigma^2-\hat\sigma_1^2}{(1-\hat\sigma_1^2)_+}\geq1$. We study from now the case where $\hat t_1>0$ and $\sigma^2\leq1/4$, with $t\leq \hat t_1$ and $d\in\{1,\hdots, d_0\}$.
We do this by showing for every dimension $d\in\{1,\hdots, d_0\}$, $\Delta_{d, t}(\hat\Sigma)$ is non-increasing on  $[0, \hat t_1]$. This implies
$$\Delta_{d_0, T-\hat\delta_{d_0}}(\hat\Sigma)\leq \Delta_{d_0, \hat{t}_{d_0}}(\hat\Sigma)\leq \Delta_{d, t}(\hat\Sigma).$$

In the remainder of the proof, we show, for $d\in\{1,\hdots, d_0\}$, that $\Delta_{d, t}(\hat\Sigma)$ is non-increasing on $[0, \hat t_1]$. This is equivalent to proving that $\Delta_{d, T-t}(\hat\Sigma)$ is non-decreasing on $[T-\hat t_1, T]$. Recall that, as in the proof as Proposition \ref{prop:non-monotonic-distance}, 
\[
\Delta_{d, T-t}(\hat\Sigma) = \sum_{d'=d+1}^{d_0}\sigma^2+ \sum_{d'=1}^d\big(\sqrt{a_t^2+(1-a_t^2)\hat\sigma_{d'}^2}-\sigma\big)^2.
\]
Consider $f_d$ given by
\[
f_d(x) = \sum_{d'=1}^d\big(\sqrt{x+(1-x)\hat\sigma_{d'}^2}-\sigma\big)^2.
\]
What we want to show is equivalent to $f$ being non-decreasing on $[a_{T-\hat t_1}^2, a_T^2]$. Since $f_d$ is convex as proven in Proposition \ref{prop:non-monotonic-distance}, it is sufficient to show that $f'$ is positive at $a_{T-\hat t_1}^2$. All in all, since the derivative of $f_d$ is
\[
f_d'(x) = \sum_{d'=1}^d\Big(1-\frac{\sigma}{\sqrt{\hat\sigma_{d'}^2+(1-\hat\sigma_{d'}^2)x}}\Big)(1-\hat\sigma_{d'}^2),
\]
if we are able to show that for any $d'\leq d_0$,
\begin{equation}
\label{eq:local-inequality-1}
(1 - \frac{\sigma}{\sqrt{\hat\sigma_{d'}^2+(1-\hat\sigma_{d'}^2)a_{T-\hat t_1}^2}})(1-\hat\sigma_{d'}^2)\geq0,
\end{equation}
then 
\[
f_d'(a_{T-\hat t}^2) = \sum_{d'=1}^{d}\Big(1 - \frac{\sigma}{\sqrt{\hat\sigma_{d'}^2 + (1 - \hat\sigma_{d'}^2)a_{T-\hat t_1}^2}}\Big)(1-\hat\sigma_{d'}^2) \geq 0.
\]
The result above is twofold. First, we get that $\Delta_{d, T-t}(\hat\Sigma)$ is increasing on the interval of interest. This also interestingly shows that the minimum of the Frobenius distance $t \mapsto \Delta_{d_0, t}(\hat\Sigma)$ is reached after $\hat t_1$. Since by definition the minium is reached at $T - \hat\delta_{d_0}$, we get that $T - \hat\delta_{d_0} \geq \hat t$.

The only thing remaining is to show~\eqref{eq:local-inequality-1}. Recall that $\hat t_1 = T-\bar{a}^{-2}\Big(\frac{4\sigma^2 - \hat\sigma_{1}^2}{1 - \hat\sigma_{1}^2}\Big)$, and that we assumed $\hat t_1>0$, which implies $\frac{4\sigma^2-\hat\sigma_{1}^2}{1-\hat\sigma_{1}^2}<a_T^2$. On the other hand, recall that we work under the event that $|\sigma^2-\hat \sigma_{1}^2|\leq\sigma^2$. Hence, $\hat\sigma_{1}^2 \leq 2 \sigma^2 < 1$ and $\frac{4\sigma^2 - \hat\sigma_{1}^2}{1 - \hat\sigma_{1}^2}>0$. Therefore, by definition of $\hat t_1$, we have
\[
a_{T-\hat t_{1}}^2 = \frac{4\sigma^2 - \hat\sigma_{1}^2}{1-\hat\sigma_{1}^2}.
\]
From here, we prove \eqref{eq:local-inequality-1}. We rewrite \eqref{eq:local-inequality-1} as 
\begin{align*}
&1 - \frac{\sigma}{\sqrt{\hat\sigma_{d'}^2 + (1 - \hat\sigma_{d'}^2)a_{T-\hat t}^2}} \geq 0.\\
\Leftrightarrow&\quad \sigma^2 \leq \hat\sigma_{d'}^2 + (1 - \hat\sigma_{d'}^2)\frac{4\sigma^2 - \hat\sigma_{1}^2}{1 - \hat\sigma_{1}^2}\\
\Leftrightarrow&\quad \sigma^2 \leq \hat\sigma_{d'}^2 + (1 - \hat\sigma_{d'}^2)\Big(1 - \frac{1 - 4\sigma^2}{1 - \hat\sigma_{1}^2}\Big)\\
\Leftrightarrow&\quad\sigma^2\leq1 - (1-\hat\sigma_{d'}^2)\frac{1 - 4\sigma^2}{1-\hat\sigma_{1}^2}\\
\Leftrightarrow&\quad(1-\hat\sigma_{d'}^2)\frac{1 - 4\sigma^2}{1-\hat\sigma_{1}^2} \leq 1 - \sigma^2.
\end{align*}
Therefore, since $\hat\sigma_{d'}<1$,we deduce that \eqref{eq:local-inequality-1} is equivalent to showing:
\[
\frac{1-4\sigma^2}{1-\hat\sigma_{1}^2} \leq \frac{1 - \sigma^2}{1 - \hat\sigma_{d'}^2}.
\]
To show this, recall the bound $\hat\sigma_{1}^2\leq2\sigma^2 < 1$. Thus, 
\[\frac{1-4\sigma^2}{1-\hat\sigma_{1}^2}\leq\frac{1-4\sigma^2}{1-2\sigma^2} = 1 - \sigma^2\frac{2}{1-2\sigma^2}\leq 1 - \sigma^2\leq\frac{1 - \sigma^2}{1 - \hat\sigma_{d'}^2},\]
which derives the desired inequality and we conclude the proof.

\subsection{Proof of Proposition \ref{prop:general-independent-gaussian-min-wasserstein}}
The proof follows by observing that the covariance matrix of $OP_dO^\top\bwdX t$ is given by
\[
\textnormal{cov}[OP_d^\top P_dO^\top\bwdX t] = O\textnormal{diag}(a_{T-t}^2+b_{T-t}^2\sigma_1^2,\hdots,a_{T-t}^2+b_{T-t}^2\sigma_d^2, 0,\hdots, 0)O^\top.
\]
Therefore, we have the following explicit form of the Fr\'echet distance between $OP_d^\top P_dO^\top\bwdX t$ and $\fwdX0$:
\[
d_F(OP_d^\top P_dO^\top\bwdX t, \fwdX0) = \sum_{j=1}^D \sigma_j^2 + \sum_{j=1}^d (a_{T-t}^2+b_{T-t}^2\sigma_j^2) - 2\sum_{j=1}^d\sigma_j\sqrt{a_{T-t}^2+b_{T-t}^2\sigma_j^2}.
\]
The proof is concluded by using the same argument as in the proof of Proposition \ref{prop:min_wasserstein_projected}.
\subsection{Proof of Proposition \ref{prop:general-gaussian-min-wasserstein}}
Recall that $\hat\Lambda=\textnormal{diag}(\hat\sigma_1^2,\hdots\hat\sigma_D^2)$ the matrix of eigenvalues of the estimated covariance matrix $\hat\Sigma=\frac1n\sum_{i=1}^nX_iX_i^\top$. We first remark that
\begin{align*}
\textnormal{cov}[\hat OP_d^\top P_d\hat O^\top\bwdhatX t] &= \hat O\textnormal{diag}(a_{T-t}^2 + b_{T-t}^2 \hat\sigma_1^2,\hdots, a_{T-t}^2 + b_{T-t}^2\hat\sigma_d^2, 0,\hdots, 0)\hat O^\top \\
&= \hat O(a_{T-t}^2P_d^\top P_d + b_{T-t}^2P_d^\top P_d\hat\Lambda)\hat O^\top.    
\end{align*}
Denote the covariance matrix of $\hat OP_d^\top P_d\hat O^\top\bwdX t$ by $\hat\Sigma_d(t)$. Recall that the Fr\'echet distance between two centered Gaussian distributions is
\[
d_F^2(\cN(0, \Sigma_1), \cN(0, \Sigma_2)) = \textnormal{tr}(\Sigma_1+\Sigma_2 - 2(\Sigma_2^{1/2}\Sigma_1\Sigma_2^{1/2})^{1/2}).
\]
In the case of interest for us, we get
\[
d_F^2(\hat OP_d^\top P_d\hat O^\top\bwdhatX t, \fwdX t) = \sum_{d'=1}^D\sigma_{d'}^2 + \sum_{d'=1}^d (a_{T-t}^2 + b_{T-t}^2\hat\sigma_{d'}^2) - 2\tr((\hat\Sigma_d^{1/2}(t)\Sigma\hat\Sigma_d^{1/2}(t))^{1/2}) .
\]
We now argue that $\tr((\hat\Sigma_d^{1/2}(t)\Sigma\hat\Sigma_d^{1/2}(t))^{1/2})$ is approximately $\sum_{d'=1}^d\hat\sigma_{d'}\sqrt{a_{T-t}^2 + b_{T-t}^2\hat\sigma_{d'}^2}$. Observe that the two quantities are equal when $\Sigma$ and $\hat\Sigma$ commute, which was the case in the previous sections where we assumed that both matrices were diagonal.  By Proposition \ref{prop:estimation-error-gaussian}, with probability $1 - 2e^{-u}$, we have $\Sigma\preceq\frac1{1-\varepsilon_u}\hat\Sigma$, where $\preceq$ denotes the Loewner order \citep[see, for instance, ][Definition 7.7.1]{horn2012matrix}. Hence,
\[
\hat\Sigma_d^{1/2}(t)\Sigma\hat\Sigma_d^{1/2}(t)\preceq\frac{1}{1-\varepsilon_u}\hat\Sigma_d^{1/2}(t)\hat\Sigma\hat\Sigma_d^{1/2}(t),
\]
by Lemma \ref{lem:matrix-properties} (i).
Since square root is a matrix monotonic function (see Lemma \ref{lem:matrix-properties} (ii)), we derive that 
\begin{align*}
\tr((\hat\Sigma_d^{1/2}(t)\Sigma\hat\Sigma_d^{1/2}(t))^{1/2})&
\leq\sqrt{\frac{1}{1-\varepsilon_u}}\tr((\hat\Sigma_d^{1/2}(t)\hat\Sigma\hat\Sigma_d^{1/2}(t))^{1/2})\\
&\leq (1+\varepsilon_u)\tr((\hat\Sigma_d^{1/2}(t)\hat\Sigma\hat\Sigma_d^{1/2}(t))^{1/2}),
\end{align*}
where we use $\varepsilon_u\leq1/2$ in the last inequality. Then, by the commutativity of $\hat\Sigma_d(t)$ and $\hat\Sigma$,
\begin{align*}
\tr((\hat\Sigma_d^{1/2}(t)&\hat\Sigma\hat\Sigma_d^{1/2}(t))^{1/2}) \\
&= \tr(\hat O(a_{T-t}^2P_d^\top P_d + b_{T-t}^2P_d^\top P_d\hat\Lambda)^{1/4}\hat\Lambda^{1/2}(a_{T-t}^2P_d^\top P_d + b_{T-t}^2P_d^\top P_d\hat\Lambda)^{1/4}\hat O^\top)\\
&=  \tr(\hat O\hat\Lambda^{1/2}(a_{T-t}^2P_d^\top P_d + b_{T-t}^2P_d^\top P_d\hat\Lambda)^{1/2}\hat O^\top) \\
&= \sum_{d'=1}^d \hat\sigma_{d'}\sqrt{a_{T-t}^2+ b_{T-t}^2\hat\sigma_{d'}^2}.
\end{align*}
By combining the results, we obtain
\[
\tr((\hat\Sigma_d^{1/2}(t)\Sigma\hat\Sigma_d^{1/2}(t))^{1/2})\leq(1+\varepsilon_u)\sum_{d'=1}^d \hat\sigma_{d'}\sqrt{a_{T-t}^2+ b_{T-t}^2\hat\sigma_{d'}^2}.
\]
We may use the same argument to derive a similar lower bound, and thus deduce that
\begin{align*}
\Big|\tr((\hat\Sigma_d^{1/2}(t)\Sigma\hat\Sigma_d^{1/2}(t))^{1/2})^{1/2}) - \sum_{d'=1}^d\hat\sigma_{d'}\sqrt{a_{T-t}^2+ b_{T-t}^2\hat\sigma_{d'}^2}\Big|& \leq \varepsilon_u\sum_{d'=1}^d \hat\sigma_{d'}\sqrt{a_{T-t}^2+ b_{T-t}^2\hat\sigma_{d'}^2}.
\end{align*}
Note that if $\hat\sigma_{d'}\geq1$, then $\sqrt{a_{T-t}^2+b_{T-t}^2\hat\sigma_{d'}^2}\leq \hat\sigma_{d'}$. Hence $\hat\sigma_{d'}\sqrt{a_{T-t}^2+ b_{T-t}^2\hat\sigma_{d'}^2}\leq\hat\sigma_{d'}^2$. On the other hand, if $\hat\sigma_d<1$, then $\sqrt{a_{T-t}^2+b_{T-t}^2\hat\sigma_{d'}^2}\leq 1$ and $\hat\sigma_{d'}\sqrt{a_{T-t}^2+ b_{T-t}^2\hat\sigma_{d'}^2}\leq\hat\sigma_{d'}$. Therefore, by recalling that $S(\Sigma) = \sum_{d'=1}^D\max(\hat\sigma_{d'}, \hat\sigma_{d'}^2)$, we deduce that
\[
\Big|\tr((\hat\Sigma_d^{1/2}(t)\Sigma\hat\Sigma_d^{1/2}(t))^{1/2})^{1/2}) - \sum_{d'=1}^d\hat\sigma_{d'}\sqrt{a_{T-t}^2+ b_{T-t}^2\hat\sigma_{d'}^2}\Big|\leq S(\Sigma)\varepsilon_u.
\]

The Fr\'echet distance $d_F(\hat OP_d^\top P_{d}\hat O^\top\bwdhatX t, \fwdX t)$ may now be bounded by
\begin{align*}
\Big|d_F^2(\hat OP_d^\top P_d\hat O^\top\bwdhatX t, \fwdX t) - \Big(\sum_{d'=1}^D\sigma_{d'}^2 + \sum_{d'=1}^d (a_{T-t}^2 + b_{T-t}^2\hat\sigma_{d'}^2) - 2\sum_{d'=1}^d\hat\sigma_{d'}&\sqrt{a_{T-t}^2+b_{T-t}^2\hat\sigma_{d'}^2} \Big)\Big|\\
&\leq  2 S(\Sigma)\varepsilon_u. 
\end{align*}

Hence, for $d\in\{2,\hdots,D\}$,
\begin{align*}
\Big|d_F^2(\hat OP_d^\top P_d\hat O^\top\bwdhatX t, \fwdX t) - d_F^2(\hat O&P_{d-1}^\top P_{d-1}\hat O^\top\bwdhatX t, \fwdX t) \\
&- \sqrt{a_{T-t}^2 + b_{T-t}^2\hat\sigma_d^2}(\sqrt{a_{T-t}^2 + b_{T-t}^2\hat\sigma_d^2} - 2\hat\sigma_d)\Big|
\leq 4S(\Sigma)\varepsilon_u.
\end{align*}
We show in the following that if $t\geq \hat T_d(u)$, then \[d_F^2(\hat OP_d^\top P_d\hat O^\top\bwdhatX t, \fwdX t)\leq d_F^2(\hat OP_{d-1}^\top P_{d-1}\hat O^\top\bwdhatX t, \fwdX t).\]
Observe that,
\begin{align}
d_F^2(\hat OP_d^\top P_d\hat O^\top\bwdhatX t, \fwdX t)&\leq d_F^2(\hat OP_{d-1}^\top P_{d-1}\hat O^\top\bwdhatX t, \fwdX t)\nonumber\\
&\quad+b_{T-t}^2\hat\sigma_d^2 + a_{T-t}^2 - 2\hat\sigma_d\sqrt{b_{T-t}^2\hat\sigma_d^2 + a_{T-t}^2} + 4S(\Sigma)\varepsilon_u\nonumber\\
&=d_F^2(\hat OP_{d-1}^\top P_{d-1}\hat O^\top\bwdhatX t, \fwdX t)\nonumber\\
&\quad+(\sqrt{a_{T-t}^2(1-\hat\sigma_d^2) + \hat\sigma_d^2} - \hat\sigma_d)^2-\hat\sigma_d^2+4S(\Sigma)\varepsilon_u.~\label{eq:local-condition-bound}
\end{align}
Hence, for $t$ such that the last term \eqref{eq:local-condition-bound} is non-positive, we have $d_F^2(\hat OP_d^\top P_d\hat O^\top\bwdhatX t, \fwdX t)\leq d_F^2(\hat OP_{d-1}^\top P_{d-1}\hat O^\top\bwdhatX t, \fwdX t)$. We now show that this is true when $t\geq \hat T_d(u)$. To do so, we split our argument in two cases. We first consider the scenario where $\hat\sigma_d\geq1$. In this case, by definition, $\hat T_d(u)=0$ and therefore we prove the result holds for all $t\in[0, T]$. Observe that $\sqrt{a_{T-t}^2(1-\hat\sigma_d^2) + \hat\sigma_d^2}\in[1, \hat\sigma_d]$, therefore
\[
\eqref{eq:local-condition-bound}\leq (1-\hat\sigma_d)^2-\hat\sigma_d^2+4S(\Sigma)\varepsilon_u = 1 - 2\hat\sigma_d + 4S(\Sigma)\varepsilon_u\leq1-2\hat\sigma_d+\hat\sigma_d\leq 0,
\]
where the last inequality holds for sufficiently small $\varepsilon_u$.

Now we consider the case where $\hat\sigma_d<1$, and hence $\sqrt{a_{T-t}^2(1-\hat\sigma_d^2)+\hat\sigma_d^2}\geq \hat\sigma_d$. Therefore,
\[
\eqref{eq:local-condition-bound}\leq0\Leftrightarrow\sqrt{a_{T-t}^2(1-\hat\sigma_d^2) + \hat\sigma_d^2}\leq\hat\sigma_d + \sqrt{\hat\sigma_d^2-4S(\Sigma)\varepsilon_u}.
\]
By squaring both sides and rearranging the terms, we deduce that
\[
\eqref{eq:local-condition-bound}\leq0\Leftrightarrow a_{T-t}^2\leq\frac{\hat\sigma_d^2 - 4S(\Sigma)\varepsilon_u + 2\hat\sigma_d\sqrt{\hat\sigma_d^2 - 4S(\Sigma)\varepsilon_u}}{1 - \hat\sigma_d^2},
\]
and we conclude by observing that the last inequality is equivalent to $t\geq \hat T_d(u)$. We derive with a similar argument that if $t\leq \hat t_d(u)$ then
\[
d_F^2(\hat OP_d^\top P_d\hat O^\top\bwdhatX t, \fwdX t)\geq d_F^2(\hat OP_{d-1}^\top P_{d-1}\hat O^\top\bwdhatX t, \fwdX t),
\]
and we conclude the proof.
\paragraph{Remark.} We can again in this case consider the same scenario as in Proposition \ref{cor:isotropic_case-optimal-stopping-time} where the eigenvalues of the covariance matrix are equal. This can be an interesting direction for future work, as to generalize the previous results to this more general setup.

\subsection{Proof of Proposition \ref{prop:optimal-constrained-score}}
We begin by rewriting the expression of the score matching objective in the following form: 
\begin{align*}
\cR(s_M) &= \frac{1}{n} \sum_{i=1}^n \E_{t \sim \mathcal{T}, \varepsilon \sim \mathcal{N}(0, I_D)}\left\|s_M(b_t X_i+a_t\varepsilon, t)+\frac{\varepsilon}{a_t}\right\|^2 \\
&= \frac{1}{n} \sum_{i=1}^n \E_{t \sim \mathcal{T}, \varepsilon}\left\| -M(t)(b_t X_i+a_t\varepsilon)+\frac{\varepsilon}{a_t}\right\|^2 \quad (\textnormal{since } s_M(x,t) = -M(t)x) \\
&= \frac{1}{n} \sum_{i=1}^n \sum_{d=1}^D\E_{t\sim\mathcal{T}, \varepsilon_d}\bigg[\Big(m_d(t) (b_t X_{ik} +a_t\varepsilon_d)-\frac{\varepsilon_d}{a_t}\Big)^2\bigg].
\end{align*}

To find the optimal $M(t)$, we note that the objective and the constraint are separable across the time interval $[0, T]$. The objective is also separable across the dimensions $d \in \{1, \dots, D\}$. Hence it suffices to minimize the quantity
\[
r(m_d(t)) := \frac{1}{n} \sum_{i=1}^n \E_{\varepsilon_d}\bigg[\Big(m_d(t) (b_t X_{ik} +a_t\varepsilon_d)-\frac{\varepsilon_d}{a_t}\Big)^2\bigg]
\]
separately over $m_d(t) \in [-C, C]$ for each $t\in[0, T]$ and $d \in \{1, \dots, D\}$. Observe that the function $r: [-C, C] \to \mathbb{R}$ is a quadratic function. Its derivative is
\begin{align*}
r'(m) &= \frac{1}{n} \sum_{i=1}^n \E_{\varepsilon_d}\left[2\left(m (b_t X_{id} +a_t\varepsilon_d)-\frac{\varepsilon_d}{a_t}\right)(b_tX_{id}+a_t\varepsilon_d)\right], \\
&= \frac{2}{n}\sum_{i=1}^n\Big(m (b_t^2X_{id}^2+a_t^2) - 1\Big) \quad (\textnormal{since } \E[\varepsilon_d]=0, \E[\varepsilon_d^2]=1), \\
&=m \left( a_t^2 + b_t^2 \frac{1}{n}\sum_{i=1}^n X_{id}^2 \right) - 1.
\end{align*}
Therefore, the minimum of $r$ over $[-C, C]$ is attained at 
\[
\hat m_d(t) = \min\left(C, \frac{1}{a_t^2+b_t^2\hat\sigma_d^2}\right) ,
\]
which concludes the proof.
\subsection{Proof of Proposition \ref{prop:learned_score_optimality}}
Since the Fr\'echet distance is determined by the variance for centered random variables, the first step of the proof is to deduce the variance of $P_d\bwdtildeX 0$ for all $d$. Denote for $1\leq d\leq D$, the variance of $\bwdtildeX{t, d}$ by $V_{t, dd}$. It is known \citep[see, for instance,][Section 5.5]{sarkka2019applied} that $V_{t, dd}$ follows the following ODE:
\begin{equation}
\label{eq:variance-ode}
\frac{dV_{t, dd}}{dt} = 2(1-2\hat m_{d}(T-t))V_{t, dd} + 2,\quad V_{0, dd} = 1.
\end{equation}
An important intermediate step in this proof is to show the following:
\begin{align}
(\sqrt{V_{t, dd}}-\sigma_d)^2\leq \sigma_d^2&\quad\textnormal{, if}\quad d\leq d_1,\label{eq:case-less-than-d1}\\
(\sqrt{V_{t, dd}}-\sigma_d)^2\geq\sigma_d^2&\quad\textnormal{, if}\quad d\geq d_2.\label{eq:case-more-than-d2}
\end{align}

To do so, we first develop an explicit expression for $V_{t, dd}$:
\begin{equation}
\label{eq:solution-variance-backward}
V_{t, dd} = \exp\left(\int_0^t2(1-2\hat m_d(T-\tau))d\tau\right) + 2\int_0^t\exp\left(\int_s^t2(1-2\hat m_d(T-\tau))d\tau\right)ds.
\end{equation}
If $C < \frac{1}{a_0^2+b_0^2\hat\sigma_d^2}=\frac{1}{\hat\sigma_d^2}$, let $t_d'$ be the unique solution in $[0, T]$ of the equation $C = \hat m_d(T - t_d') = \frac{1}{a_{T-t_d'}^2+b_{T-t_d'}^2\hat\sigma_d^2}$. Otherwise, we set $t_d' = T$, which is always the case for $d\leq d_1$. Remark that if $\hat\sigma_d\geq1$, then $\frac{1}{a_t^2+b_t^2\hat\sigma_d^2}\leq 1\leq C$. Thus, for such dimension $d$, we always have $t_d'=T$ and $d\leq d_1$.

We derive an explicit expression for the term $V_{T, dd}$. We first calculate the first part by plugging in the exact form of $\hat m_d$. To do so, we recall that $a_t = \sqrt{1-e^{-2t}}$ and $b_t = e^{-t}$. Also note that $\hat m_d$ is decreasing on $[0,T]$, more precisely it is equal to $C$ on $[0, T - t_d']$ and equal to $1/(a_t^2 + b_t^2 \hat \sigma_d^2)$ for $t \in [T-t_d', T]$. With these keys facts in mind, we begin by calculating the following integrand, which for $s=0$ gives the first term in \eqref{eq:solution-variance-backward} and is the integrand of the second term.
\begin{align}
\exp&\left(\int_s^T2(1-2\hat m_d(T-\tau))d\tau\right)\nonumber\\
&= \exp\left(\int_0^{T-s}2(1-2\hat m_d(\tau))d\tau\right)\nonumber\\
&=e^{2(T-s)}\exp\left(-4\int_0^{T - s\vee t_d'}\hat m_d(\tau)d\tau\right)\exp\left(-4\int_{T - s\vee t_d'}^{T-s}\hat m_d(\tau)d\tau\right)\nonumber\\
&=e^{2(T-s)}e^{-4C(T - s\vee t_d')}e^{-4(s\vee t_d'-s)}\left(\frac{1-(1-\hat\sigma_d^2)e^{-2(T-s\vee t_d')}}{1-(1-\hat\sigma_d^2)e^{-2(T-s)}}\right)^2~\label{eq:local-integral},
\end{align}
where, in the last line, we use the following:
\[
\int\hat m_d(\tau)d\tau = \int \left(1 + \frac{e^{-2\tau}(1-\hat\sigma_d^2)}{1-(1-\hat\sigma_d^2)e^{-2\tau}}\right)d\tau = \tau + \frac{1}{2}\log\left(1 - (1-\hat\sigma_d^2)e^{-2\tau}\right).
\]
By substituting $s=0$, we see that the first term in \eqref{eq:solution-variance-backward} is equal to
\begin{equation}
\label{eq:local-first-term}
\exp\left(\int_0^T2(1-2\hat m_d(T-\tau))d\tau\right) = e^{-2T}e^{-4(C-1)(T-t_d')}\left(\frac{1-(1-\hat\sigma_d^2)e^{-2(T-t_d')}}{1-(1-\hat\sigma_d^2)e^{-2T}}\right)^2.
\end{equation}
Next, we focus on deriving an explicit expression of the second term in \eqref{eq:solution-variance-backward}. We plug in the term \eqref{eq:local-integral} and deduce that
\begin{align*}
2&\int_0^T\exp\left(\int_s^T2(1-2\hat m_d(T-\tau))d\tau\right)ds\\
&=2(\int_{t_d'}^T+\int_0^{t_d'})e^{2(T-s)}e^{-4C(T - s\vee t_d')}e^{-4(s\vee t_d'-s)}\left(\frac{1-(1-\hat\sigma_d^2)e^{-2(T-s\vee t_d')}}{1-(1-\hat\sigma_d^2)e^{-2(T-s)}}\right)^2ds\\
&=2\int_{t_d'}^T e^{2(T-s)}e^{-4C(T-s)}ds\\
&\qquad+ 2\int_0^{t_d'}e^{2(T-s)}e^{-4C(T-t_d')}e^{-4(t_d'-s)}\left(\frac{1-(1-\hat\sigma_d^2)e^{-2(T-t_d')}}{1-(1-\hat\sigma_d^2)e^{-2(T-s)}}\right)^2ds\\
&=\frac{1}{2C-1}(1-e^{(2-4C)(T-t_d')}) \\
&\qquad+ (1-(1-\hat\sigma_d^2)e^{-2(T-t_d')})^2e^{-4(C-1)(T-t_d')}\int_0^{t_d'}\frac{2e^{-2(T-s)}}{(1-(1-\hat\sigma_d^2)e^{-2(T-s)})^2}ds\\
&=\frac{1}{2C-1}(1-e^{(2-4C)(T-t_d')}) \\
&\qquad+ \frac{(1-(1-\hat\sigma_d^2)e^{-2(T-t_d')})^2e^{-4(C-1)(T-t_d')}}{1 - \hat\sigma_d^2}\left[\frac{1}{1-(1-\hat\sigma_d^2)e^{-2(T-s)}}\right]_0^{t_d'}.
\end{align*}
We see that the last term can be rewritten in the following form
\begin{align*}
\left[\frac{1}{1-(1-\hat\sigma_d^2)e^{-2(T-s)}}\right]_0^{t_d'} &= \frac{1}{1-(1-\hat\sigma_d^2)e^{-2(T-t_d')}}-\frac{1}{1-(1-\hat\sigma_d^2)e^{-2T}}\\
&=\frac{(1-\hat\sigma_d^2)(e^{-2(T-t_d')}-e^{-2T})}{(1-(1-\hat\sigma_d^2)e^{-2T})(1-(1-\hat\sigma_d^2)e^{-2(T-t_d')})}.
\end{align*}
Thus, we derive that
\begin{align}
2\int_0^T&\exp\left(\int_s^T2(1-2\hat m_d(T-\tau))d\tau\right)ds\nonumber\\ 
&= \frac{1}{2C-1}(1-e^{(2-4C)(T-t_d')})\nonumber \\
&\qquad+ \frac{(1-(1-\hat\sigma_d^2)e^{-2(T-t_d')})e^{-4(C-1)(T-t_d')}}{1-(1-\hat\sigma_d^2)e^{-2T}}(e^{-2(T-t_d')}-e^{-2T})~\label{eq:local-second-term}.
\end{align}
Therefore, by summing up the two terms \eqref{eq:local-first-term} and \eqref{eq:local-second-term}, we deduce that
\begin{align*}
V_{T, dd} &=\frac{1}{2C-1}(1-e^{(2-4C)(T-t_d')}) \\
&\quad+ \frac{(1-(1-\hat\sigma_d^2)e^{-2(T-t_d')})e^{-4(C-1)(T-t_d')}}{1-(1-\hat\sigma_d^2)e^{-2T}}(e^{-2(T-t_d')}-e^{-2T})\\
&\quad+e^{-2T}e^{-4(C-1)(T-t_d')}\left(\frac{1-(1-\hat\sigma_d^2)e^{-2(T-t_d')}}{1-(1-\hat\sigma_d^2)e^{-2T}}\right)^2\\
&=\frac{1}{2C-1}(1-e^{(2-4C)(T-t_d')}) \\
&\quad+ e^{-4(C-1)(T-t_d')}\frac{1-(1-\hat\sigma_d^2)e^{-2(T-t_d')}}{1-(1-\hat\sigma_d^2)e^{-2T}}\\
&\quad\times\left(e^{-2(T-t_d')}-e^{-2T} + e^{-2T}\frac{1-(1-\hat\sigma_d^2)e^{-2(T-t_d')}}{1-(1-\hat\sigma_d^2)e^{-2T}}\right)\\
&=\frac{1}{2C-1}(1-e^{(2-4C)(T-t_d')})\\
&\quad+ e^{-4(C-1)(T-t_d')}\frac{1-(1-\hat\sigma_d^2)e^{-2(T-t_d')}}{1-(1-\hat\sigma_d^2)e^{-2T}}\\
&\quad\times\frac{e^{-2(T-t_d')}-2(1-\hat\sigma_d^2)e^{-2(2T-t_d')} + (1-\hat\sigma_d^2)e^{-4T}}{1-(1-\hat\sigma_d^2)e^{-2T}}\\
&=\frac{1}{2C-1}(1-e^{(2-4C)(T-t_d')})\\
&\quad+ e^{(2-4C)(T-t_d')}\frac{1-(1-\hat\sigma_d^2)e^{-2(T-t_d')}}{1-(1-\hat\sigma_d^2)e^{-2T}}\frac{1-2(1-\hat\sigma_d^2)e^{-2T} + (1-\hat\sigma_d^2)e^{-2(T + t_d')}}{1-(1-\hat\sigma_d^2)e^{-2T}}.
\end{align*}
We remark that, for $d\leq d_1$ we have $t_d'=T$ and we may simplify the expression of $V_{T, dd}$ to
\begin{equation}    \label{eq:tech-2}
V_{T, dd} = \hat\sigma_d^2\frac{1 - 2(1-\hat\sigma_d^2)e^{-2T}+(1-\hat\sigma_d^2)e^{-4T}}{1 - 2(1-\hat\sigma_d^2)e^{-2T}+(1-\hat\sigma_d^2)^2e^{-4T}}.    
\end{equation}
Before we prove \eqref{eq:case-less-than-d1} and \eqref{eq:case-more-than-d2}, we categorize the behavior of $V_{t, dd}$ according to the value of $\hat\sigma_d$ and we summarize the result in the following lemma, the proof of which  we delay to the end of this proof.
\begin{lemma} \label{lem:variance-bound}
For $d\in\{1,\hdots, D\}$. If $\hat\sigma_d\geq1$, then $V_{t, dd}\geq1$ for every $t\in[0, T]$. If $\hat\sigma_d\leq1$, then $V_{t, dd}\leq1$ for every $t\in[0, T]$.
\end{lemma}

Let us deduce from \eqref{eq:tech-2}, for $d\leq d_1$, $(\sqrt{V_{T, dd}} - \sigma_d)^2 \leq\sigma_d^2$.
We work under the high probability event that $|\sigma_d^2-\hat\sigma_d^2|\leq \sigma_d^2$ for every $d\in\{1,\hdots, D\}$, we split the proof into three cases:
\begin{itemize}
    \item If $\sigma_d > 1$ then $\hat\sigma_d \geq 1$, from \eqref{eq:tech-2}, we see that $V_{T, dd}<\hat\sigma_d^2\leq 4\sigma_d^2$, with high probability.  Thus, $(\sqrt{V_{T, dd}}-\sigma_d)^2 \leq \max((0 - \sigma_d)^2,  (2\sigma_d-\sigma_d)^2)\leq\sigma_d^2$.
    \item If $\sigma_d\in[\frac12, 1)$ which implies that $\hat\sigma_d<1$, then $V_{T,dd}\leq 1$, we then have $(\sqrt{V_{T, dd}} - \sigma_d)^2 \leq\max((0-\sigma_d)^2, (1-\sigma_d)^2)\leq\sigma_d^2$.
    \item If $\sigma_d = 1$, we again split cases depending on whether $\hat\sigma_d \geq 1$. We get the same bounds as in the two previous cases.
    \item Finally, if $\sigma_d\leq\frac12$, with high probability, $|\sigma_d^2-\hat\sigma_d^2|\leq\sigma_d^2$. Hence, $\hat\sigma_d^2\leq 2\sigma_d^2\leq\frac12$. Observing that the fraction in \eqref{eq:tech-2} is bounded by $1/(1-\hat \sigma_d^2)$, we deduce that
\[
V_{T, dd} \leq \frac{\hat\sigma_d^2}{1-\hat\sigma_d^2}\leq 2\hat\sigma_d^2 \leq 4 \sigma_d^2,\quad \forall d\leq d_1,
\]
which gives the desired bound.
\end{itemize}

Next, for $d\geq d_2$, remark by definition of $t_d'$ that $\frac{1}{1-(1-\hat\sigma_d^2)e^{-2(T-t_d')}} = C$. Also note that the definition of $d_2$ and the fact that $C>1$ implies that $\hat \sigma_d^2 <1$. Hence, 
\begin{align*}
V_{T, dd} &= \frac{1}{2C-1} + e^{(2-4C)(T-t_d')}\left(\frac{(1-2(1-\hat\sigma_d^2)e^{-2T} + (1-\hat\sigma_d^2)e^{-2(T + t_d')})}{C(1-(1-\hat\sigma_d^2)e^{-2T})^2}-\frac{1}{2C-1}\right)\\
&\geq \frac{1}{2C-1} + e^{(2-4C)(T-t_d')}\left(\frac{1}{C}-\frac{1}{2C-1}\right)\\
&\geq \frac{1}{2C-1}.
\end{align*}
Therefore, for $d\geq d_2$ we deduce that
\[
V_{T, dd} \geq \frac{1}{2C-1} \geq 4\sigma_d^2.
\]
To summarize, we derived the following bounds
\[
(\sqrt{V_{T,dd}}-\sigma_d)^2\leq \sigma_d^2, \quad\forall d\leq d_1,
\]
and
\[
(\sqrt{V_{T,dd}}-\sigma_d)^2\geq \sigma_d^2, \quad\forall d> d_2.
\]
By definition of the Fr\'echet distance, we have
\[
d_F^2(P_d^\top P_d\bwdtildeX T, \fwdX 0) = \sum_{j=1}^d (\sqrt{V_{T, jj}} - \sigma_j)^2 + \sum_{j=d+1}^D\sigma_j^2,
\]
we deduce that, for any $d < d_1\leq d_2 < d'$,
\begin{align*}
d_F(P_{d_1}^\top P_{d_1}\bwdtildeX T, \fwdX 0) &= \sum_{j=1}^{d_1} (\sqrt{V_{T, jj}} - \sigma_j)^2 + \sum_{j=d_1+1}^D\sigma_j^2\\
&= \sum_{j=1}^{d} (\sqrt{V_{T, jj}} - \sigma_j)^2 + \sum_{j=d+1}^{d_1} (\sqrt{V_{T, jj}} - \sigma_j)^2 + \sum_{j=d_1+1}^D\sigma_j^2\\
&\leq \sum_{j=1}^{d} (\sqrt{V_{T, jj}} - \sigma_j)^2 + \sum_{j=d+1}^D\sigma_j^2\\
&=d_F(P_d^\top P_d\bwdtildeX T, \fwdX 0),
\end{align*}
and
\begin{align*}
d_F(P_{d_2}^\top P_{d_2}\bwdtildeX T, \fwdX 0) &= \sum_{j=1}^{d_2} (\sqrt{V_{T, jj}} - \sigma_j)^2 + \sum_{j=d_2+1}^D\sigma_j^2 \\
&= \sum_{j=1}^{d_2} (\sqrt{V_{T, jj}} - \sigma_j)^2 + \sum_{j=d_2+1}^{d'}\sigma_j^2 + \sum_{j=d'+1}^{D}\sigma_j^2\\
&\leq \sum_{j=1}^{d'} (\sqrt{V_{T, jj}} - \sigma_j)^2 +  \sum_{j=d'+1}^{D}\sigma_j^2\\
&=d_F(P_{d'}^\top P_{d'}\bwdtildeX T, \fwdX 0).
\end{align*}
Therefore, the minimum of $d_F(P_d^\top P_{d}\bwdtildeX T, \fwdX 0)$ must occur between $d_1$ and $d_2$.

\paragraph{Proof of Lemma \ref{lem:variance-bound}.} Recall that $V_{t, dd}$ satisfies the ODE \eqref{eq:variance-ode}
\[
\frac{dV_{t, dd}}{dt} = 2(1 - 2\hat m_d(T-t))V_{t, dd} + 2,\quad V_{0, dd} = 1.
\]
Assume that $\hat\sigma_d>1$ and by contradiction that $V_{t, dd}<1$ for some $t\in[0, T]$. Let $t_0=\inf\{t:V_{t, dd}<1\}$, by continuity, we have $V_{t_0, dd}=1$. Then we have
\[
\Big[\frac{dV_{t, dd}}{dt}\Big]_{t=t_0} = 2(1-2\hat m_d(T-t_0))V_{t_0, dd} +2 = 2\Big(1 - \frac{2}{1-e^{-2t}+e^{-2t}\hat\sigma_d^2}\Big) + 2,
\]
where we use the fact that $V_{t_0, dd}=1$. The last term can be rewritten as
\[ \frac{4(\hat\sigma_d^2-1)e^{-2t}}{1 - e^{-2t}+e^{-2t}\hat\sigma_d^2}.
\]
Hence we have $[\frac{dV_{t, dd}}{dt}]_{t=t_0} > 0$ which contradicts the definition of $t_0$. Hence $V_{t, dd}\geq 1$ for all $t\in[0, T]$. The case for $\hat\sigma_d<1$ can be derived similarly.

\subsection{Derivation of special cases of Proposition \ref{prop:learned_score_optimality}}
\label{app:derivation-special-case-erm}
First, consider the scenario where the learning capacity is unconstrained, effectively setting $C=\infty$, while the data covariance matrix is nonsingular. In this case, the condition on $d_1$ becomes $0 \leq \hat\sigma_d^2$, which %
is trivially satisfied for all $d \in \{1, \hdots, D\}$, implying $d_1 = D$. The condition for $d_2$ becomes $0 > 4\sigma_d^2$, which holds for none of $d$, thus implying $d_2 = D$. Therefore, when $C=\infty$, Proposition \ref{prop:learned_score_optimality} entails that $d_{\min}=D$. This result is somewhat expected: if the score function is learned perfectly, the diffusion process can be reversed in the full ambient space, enabling sampling from the target distribution without any need for dimensionality reduction.

Second, consider the scenario addressed in Proposition \ref{cor:isotropic_case-optimal-stopping-time} where the true data distribution lies within a $d_0$-dimensional linear subspace, i.e., $\sigma_{d_0+1} = \cdots = \sigma_D = 0$ and $\sigma_1 = \cdots = \sigma_{d_0} = \sigma$. Assume that $C$ is sufficiently large to ensure that $1/C \leq \min(\sigma^2, \min_{d'\in\{1,\hdots,d_0\}}\hat\sigma_{d'}^2)$. 
Therefore, for $d\leq d_0$, one has $\frac{1}{C} \leq \hat\sigma_d^2$ (which is not satisfied anymore for $d$ beyond $d_0$), leading to $d_1=d_0$. On the other hand, for $d>d_0$ we have $\frac1{2C-1}\geq 0 = 4\sigma_d$. Hence $d_0= d_1\leq d_2\leq d_0$, which implies $d_2 = d_0$.  Thus, Proposition \ref{prop:learned_score_optimality} predicts $d_{\min} = d_0$. This suggests that the projection onto the subspace in which the data distribution lies is the optimal sampling strategy, which is in line with the recommendation of Proposition \ref{cor:isotropic_case-optimal-stopping-time}.
\paragraph{Proof of Corollary \ref{cor:special-case}.} By the definition of $d$, we have $d_1 = d$. It remains to prove that $d_2\leq d + 1$. With $n$ large enough and high probability, we have $\hat\sigma_{d+1}\geq\sigma_{d+1}^2/2$. Therefore, 
\[
\frac{1}{4(2C-1)}\geq\frac{1}{8C}\geq\frac{\hat\sigma_{d+1}^2}{8}\geq\frac{\sigma_{d+1}^2}{16}=\frac{\lambda^{-(d+1)}}{16}\geq \lambda^{-(d+2)},
\]
where we use the fact that $\lambda\geq16$. This shows that $d_2< d+2$. Hence $d_2\leq d+1$.

\section{Bounds on Gaussian estimation}
In this section, we give some bounds for the estimation error for Gaussian distributions.
\begin{proposition}
\label{prop:estimation-error-independent-gaussian}
Let $(X_1,\hdots,X_n)$ be sample drawn independently from $\cN(0, \sigma^2)$. Then, for $\varepsilon>0$, we have
\[
\PP\Big[\big|\frac{1}{n}\sum_{i=1}^nX_i^2-\sigma_d^2\big|\leq\varepsilon\sigma_d^2\Big]\geq 1-2\exp\Big(-\frac{\varepsilon^2n}{4(\varepsilon+1)}\Big).
\]
\end{proposition}
\begin{proof}
By \citet{ghosh2021exponential}, if $Z \sim \chi^2(p)$ and $u>0$,
\[
\PP[|Z-p| \geq u] \leq 2 \exp \Big(- \frac{u^2}{4(p+u)}\Big).
\]
The result then unfolds from standard manipulations after observing that $\frac{1}{\sigma^2}\sum_{i=1}^nX_i^2$ follows a $\chi^2(n)$.
\end{proof}
\begin{proposition}
\label{prop:estimation-error-gaussian}
Let $\Sigma$ be a semi-definite positive $D\times D$ matrix, and assume the sample $(X_1, \hdots, X_n)$ is drawn independently from $\cN(0, \Sigma)$. Then, there is a universal constant $C$ such that, with probability $1 - 2e^{-u}$, the empirical covariance matrix $\hat\Sigma = \frac1n\sum_{i=1}^nX_iX_i^\top$ satisfies: 
\[
-\frac{8C}{3}(\sqrt{\frac{D+u}{n}}+\frac{D+u}{n})\Sigma\preceq\hat\Sigma - \Sigma \preceq \frac{8C}{3}\Big(\sqrt{\frac{D+u}{n}} + \frac{D+u}{n}\Big)\Sigma,
\]
where $\preceq$ denotes the Loewner order.
\end{proposition}
\begin{proof}
It is shown in \citet[Theorem 4.6.1]{vershynin2018high} that, with probability $1 - 2e^{-u}$,  
\[
\|\Sigma^{-1/2}\hat\Sigma\Sigma^{-1/2} - I_D \|_{op}\leq K^2C(\sqrt{\frac{D+u}{n}}+\frac{D+u}{n}),
\]
where $\|\cdot\|_{op}$ denotes the operator norm and $K$ is a constant satisfying
\[
\|X^\top x\|_{\psi_2} \leq K \|X^\top x\|_{L_2}, \forall x\in\R^D,
\]
where $\|X\|_{\psi_2} = \inf\{K>0:\E[e^{{X^2}/{K^2}}]\leq 2\}$. It is shown in \citet[Section 2.6.1]{vershynin2018high} that, if $X$ follows a centered Gaussian distribution with standard deviation $\sigma$, then $\|X\|_{\psi_2} = \sigma\sqrt{8/3}$ and $\|X\|_{L_2}=\sigma$.
Hence, $K=\sqrt{8/3}$ in our case and we have
\[
-\frac{8C}{3}(\sqrt{\frac{D+u}{n}}+\frac{D+u}{n})I_D\preceq \Sigma^{-1/2}\hat\Sigma\Sigma^{-1/2} - I_d\preceq \frac{8C}{3}(\sqrt{\frac{D+u}{n}}+\frac{D+u}{n})I_D.
\]
By multiplying $\Sigma^{1/2}$ from left and right for both side, we derive that
\[
-\frac{8C}{3}(\sqrt{\frac{D+u}{n}}+\frac{D+u}{n})\Sigma\preceq\hat\Sigma - \Sigma \preceq \frac{8C}{3}(\sqrt{\frac{D+u}{n}}+\frac{D+u}{n})\Sigma.
\]
\end{proof}
\section{Useful Lemma}
In this section we provide some lemma that will be useful throughout the whole paper \citep[see also,][Section 7.7]{horn2012matrix}.
\begin{lemma}
\label{lem:matrix-properties}
Let $A, B$ be two symmetric $D\times D$ real matrices, and $S$ be an arbitrary $D\times D$ real matrix. The following statements hold:
\begin{enumerate}
    \item[(i)] If $A\preceq B$, then $S^\top AS\preceq S^\top BS$.
    \item[(ii)] If $A^2\preceq B^2$, then $A\preceq B$. In particular if $A$ and $B$ are semi-definite positive, then $A\preceq B\Rightarrow \sqrt A\preceq\sqrt B$.
\end{enumerate}
\end{lemma}

\section{Experiment details}    \label{app:exp}
\subsection{Synthetic data experiment}
\paragraph{Synthetic Gaussian data.} In the experiment of Figure \ref{fig:projection_experiments}, we generate data using Gaussian distribution with covariance matrices equal to $\textnormal{diag}(1, 0.6, 0.6^2, \hdots, 0.6^6, 10^{-10}, 10^{-10})$ (left) and $\textnormal{diag}(10, 0.2, 0.2, 0.2, 0, 0)$ (right). We then generate sample by first estimating the variances with the data with $1$k sample, then solving the SDE \eqref{eq:SDE-backward-estimated} separately for each projection. We generate new sample using the Ornstein-Uhlenbeck process with $T=2$ and $1000$ discretization steps.
\paragraph{Synthetic Gaussian data with exponentially decaying eigenvalues.} In the experiment of Figure \ref{fig:linear-model-score-matching}, we generate data using Gaussian distribution with covariance matrices equal to $\textnormal{diag}(1, 1/4, 1/4^2, \hdots, 1/4^9)$. We then train a linear model with score matching using $10$ thousand samples, and we clip the parameters of the model. We generate new sample using the Ornstein-Uhlenbeck process with $T=2$ and $1000$ discretization steps.
\begin{figure}[h!]
        \centering
        \includegraphics[width=0.5\linewidth]{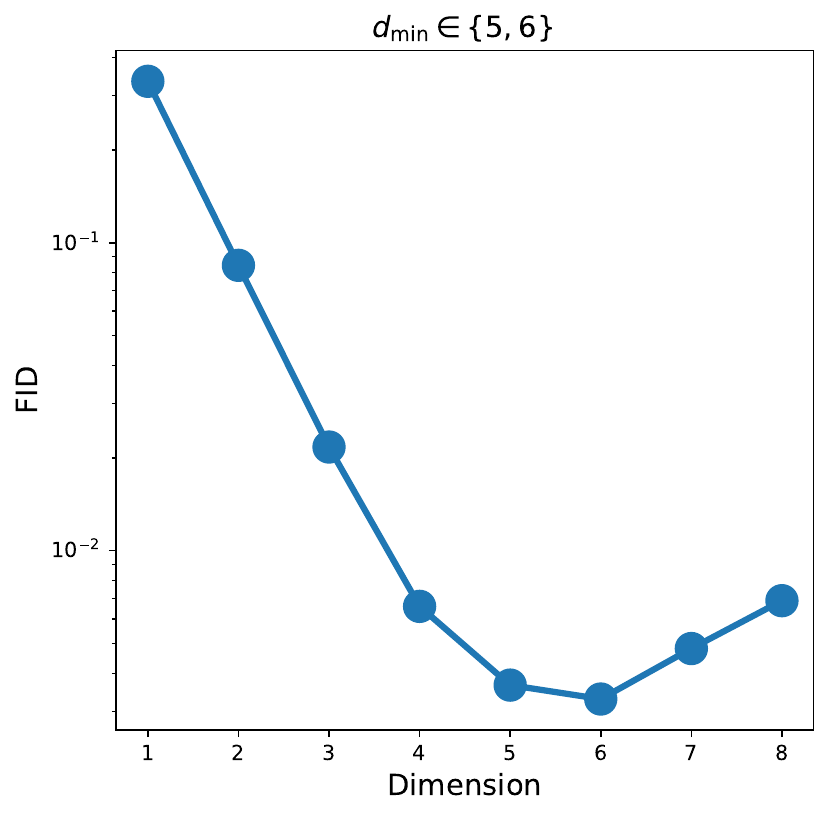}
        \caption{Fr\'echet distance of the sample generated by a diffusion model, where the data follows a $10$-dimensional Gaussian distribution with exponentially decaying eigenvalues, while the score is a linear function trained with sample projected onto a lower dimension. Corollary \ref{cor:special-case} correctly predicts the optimal projection dimension, here $d_{\min}\in\{5, 6\}$.}
\label{fig:linear-model-score-matching}
\end{figure}
\subsection{Natural image experiment}
\paragraph{Common details.} We use the dataset CelebA and CelebA-HQ \citep{liu2015faceattributes} and ImageNet-256 \citep{deng2009imagenet}. We use a U-Net model \citep{ronneberger2015u} and an Adam optimizer \citep{kingma2014adam}. The diffusion model uses variance preserving noise schedule \citep{song2020score}.  The code was implemented in JAX \citep{jax2018github}.
\paragraph{Training of AE on CelebA-HQ.} We train a VQ-VAE using the VQ-GAN loss \citep{esser2021taming} for $1.95$ million step on $20$ TPUv2. The VQ-VAE encodes the images to a latent space of shape $64\times64\times3$ and reaches a $2$k-rFID score of $2.44$. Other hyperparameters for training is summarized in Table~\ref{tab:hyperparams-AE}.
\begin{table}[ht]
    \centering
    \begin{tabular}{cc}
    \toprule
    {\bf Name} & {\bf Value} \\
    \midrule
    Coefficient of the adversarial loss & $0.1$\\
    Coefficient of the generator loss & $100$\\
    Coefficient of the LPIPS loss & $1.0$ \\
    Coefficient of the discriminator loss & $0.01$ \\
    Number of embeddings of the vector quantizer & $8192$ \\
    Optimizer & Adam with standard hyperparameters \\
    EMA decay & $0.9999$ \\
    Learning rate & $10^{-5}$ \\
    Batch size & $16$ \\
    \bottomrule
    \end{tabular}
    \caption{Hyperparameters for training VQ-VAE on CelebA-HQ.}
    \label{tab:hyperparams-AE}
\end{table}
\paragraph{Training of AE on ImageNet-256.} We train a VQ-VAE using the VQ-GAN loss \citep{esser2021taming} for $715$ thousand (resp. $1.81$ million) steps on $20$ TPUv2 for a latent shape of $32\times32\times4$ (resp. $64\times64\times3$). The VQ-VAE achieves an rFID score of $3.6$ (resp. $4.2$). Other hyperparameters for training is summarized in Table~\ref{tab:hyperparams-AE-Imagenet}.
\begin{table}[ht]
    \centering
    \begin{tabular}{cc}
    \toprule
    {\bf Name} & {\bf Value} \\
    \midrule
    Coefficient of the adversarial loss & $0.1$\\
    Coefficient of the generator loss & $100$\\
    Coefficient of the LPIPS loss & $1.0$ \\
    Coefficient of the discriminator loss & $0.01$ \\
    Number of embeddings of the vector quantizer & $8192$ \\
    Optimizer & Adam with standard hyperparameters \\
    EMA decay & $0.999$ \\
    Learning rate & $10^{-5}$ \\
    Batch size & $16$ \\
    \bottomrule
    \end{tabular}
    \caption{Hyperparameters for training VQ-VAE on ImageNet-256.}
    \label{tab:hyperparams-AE-Imagenet}
\end{table}
\paragraph{Training of LDM on CelebA-HQ.} We train an LDM on the images encoded by the AE we described above. We train for $5.25$ million steps on $8$ TPUv6. We summarize the hyperparameters used in Table \ref{tab:hyperparams-LDM}.
\begin{table}[ht]
    \centering
    \begin{tabular}{cc}
    \toprule
    {\bf Name} & {\bf Value} \\
    \midrule
    Noise schedule & Variance Preserving \\
    Number of sampling steps & $250$ \\
    Optimizer & Adam with standard hyperparameters \\
    EMA decay & $0.9999$ \\
    Learning rate & $10^{-4}$ \\
    Batch size & $16$ \\
    \bottomrule
    \end{tabular}
    \caption{Hyperparameters for training LDM on encoded images of CelebA-HQ.}
    \label{tab:hyperparams-LDM}
\end{table}
\begin{table}[ht]
    \centering
    \begin{tabular}{cc}
    \toprule
    {\bf Name} & {\bf Value} \\
    \midrule
    Noise schedule & Variance Preserving \\
    Number of sampling steps & $250$ \\
    Optimizer & Adam with standard hyperparameters \\
    EMA decay & $0.9999$ \\
    Learning rate & $10^{-4}$ \\
    Batch size & $128$ \\
    \bottomrule
    \end{tabular}
    \caption{Hyperparameters for training diffusion model on CelebA and ImageNet-64.}
    \label{tab:hyperparams-DM}
\end{table}
\paragraph{Training of LDM on ImageNet-256.} We train two LDMs on the images encoded by the AEs we described above. We train for $1$ million steps on $8$ TPUv6 for both LDMs. We summarize the hyperparameters used in Table \ref{tab:hyperparams-LDM-Imagenet}.
\begin{table}[ht]
    \centering
    \begin{tabular}{cc}
    \toprule
    {\bf Name} & {\bf Value} \\
    \midrule
    Noise schedule & Variance Preserving \\
    Number of sampling steps & $250$ \\
    Optimizer & Adam with standard hyperparameters \\
    EMA decay & $0.999$ \\
    Learning rate & $10^{-4}$ \\
    Batch size & $1024$ \\
    \bottomrule
    \end{tabular}
    \caption{Hyperparameters for training LDM on encoded images of ImageNet-256.}
    \label{tab:hyperparams-LDM-Imagenet}
\end{table}
\paragraph{Training pixel diffusion model on CelebA and ImageNet-64.} We train diffusion models directly on CelebA and ImageNet-64. We train both models for $1$ million steps on $12$ TPUv2. We summarize the hyperparameters in Table \ref{tab:hyperparams-DM}.

\paragraph{Results.} We previously introduced some results in Section \ref{sec:intro}. Here, we present additional evidence regarding the quality of the generated images. We observe (Figure \ref{fig:latent-diffusion}) that in the final few steps, the sample of LDM does not change visibly. On the contrary, the images generated in pixel space (Figure \ref{fig:pixel-diffusion}) are still denoised even in the last steps.

\paragraph{Linear AE and U-Net LDM on MNIST \citep{deng2012mnist}.} We train linear AE on MNIST with dimension $64$, $256$, and $400$ for $5$k steps each using $4$ TPUv2. Then we train U-Net diffusion models paired with each AE for $10$k steps each using $4$ TPUv2. We summarize other hyperparameters in Table \ref{tab:hyperparams-AE-MNIST} and \ref{tab:hyperparams-LDM-MNIST}.

\begin{table}[ht]
    \centering
    \begin{tabular}{cc}
    \toprule
    {\bf Name} & {\bf Value} \\
    \midrule
    Optimizer & Adam with standard hyperparameters \\
    Learning rate & $0.003$ \\
    Batch size & $256$ \\
    \bottomrule
    \end{tabular}
    \caption{Hyperparameters for training Linear AE on MNIST.}
    \label{tab:hyperparams-AE-MNIST}
\end{table}
\begin{table}[h!]
    \centering
    \begin{tabular}{cc}
    \toprule
    {\bf Name} & {\bf Value} \\
    \midrule
    Noise schedule & Variance Preserving \\
    Number of sampling steps & $250$ \\
    Optimizer & Adam with standard hyperparameters \\
    EMA decay & $0.999$ \\
    Learning rate & $10^{-4}$ \\
    Batch size & $256$ \\
    \bottomrule
    \end{tabular}
    \caption{Hyperparameters for training LDM on encoded images of MNIST.}
    \label{tab:hyperparams-LDM-MNIST}
\end{table}
\begin{figure}
    \centering
    \includegraphics[width=0.7\linewidth]{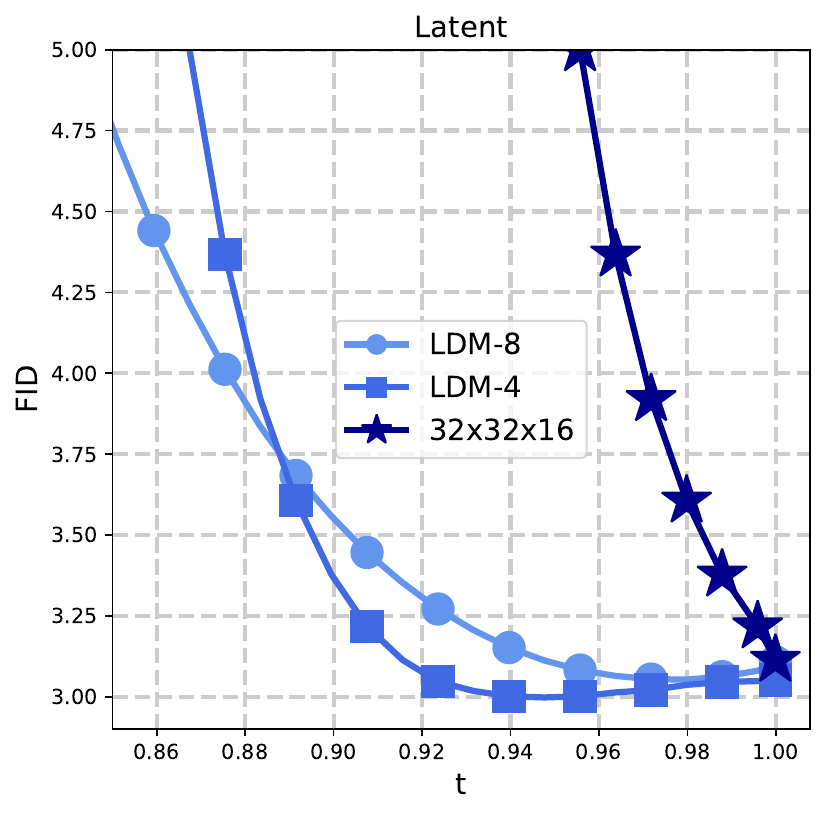}
    \caption{Zoom in of Figure \ref{fig:LDM-AE-imagenet}.}
    \label{fig:zoom-in}
\end{figure}
\begin{figure}[ht]
\centering
\hfill
\includegraphics[width=.45\linewidth]{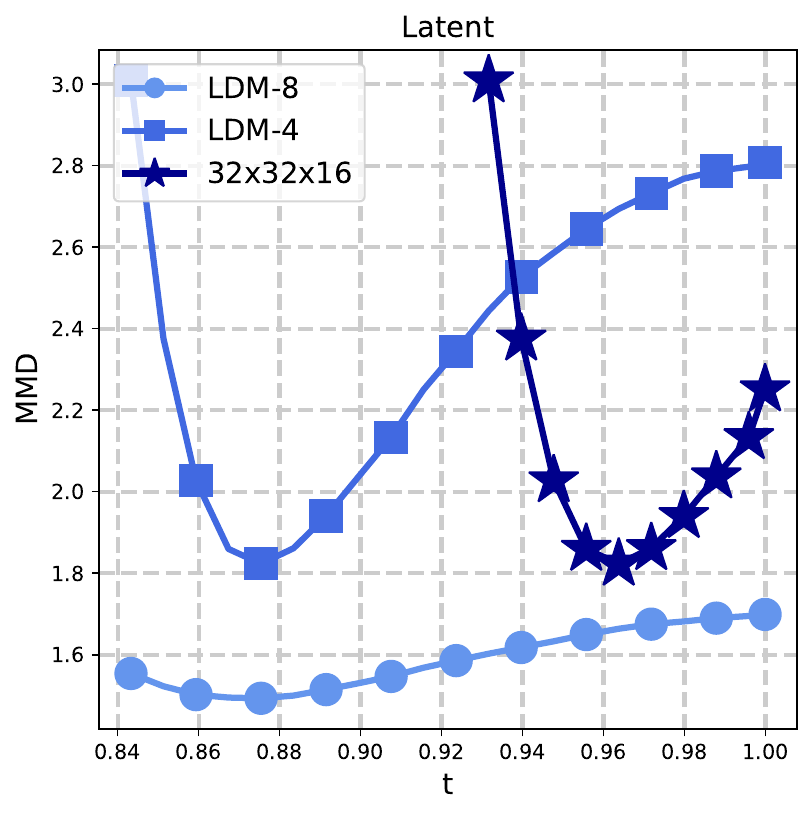}
\hfill
\includegraphics[width=.45\linewidth]{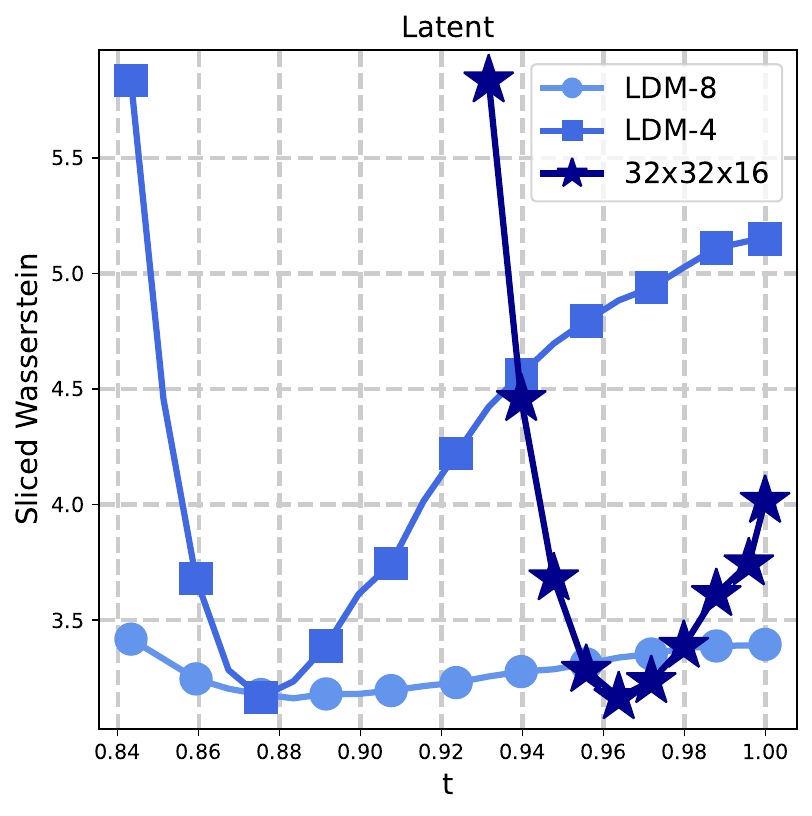}
\hfill\hfill
\caption{We measure the image quality of different LDMs on ImageNet-256 by MMD and Sliced Wasserstein distance.}
\end{figure}
\begin{figure}[ht]
    \centering
    \includegraphics[width=\linewidth]{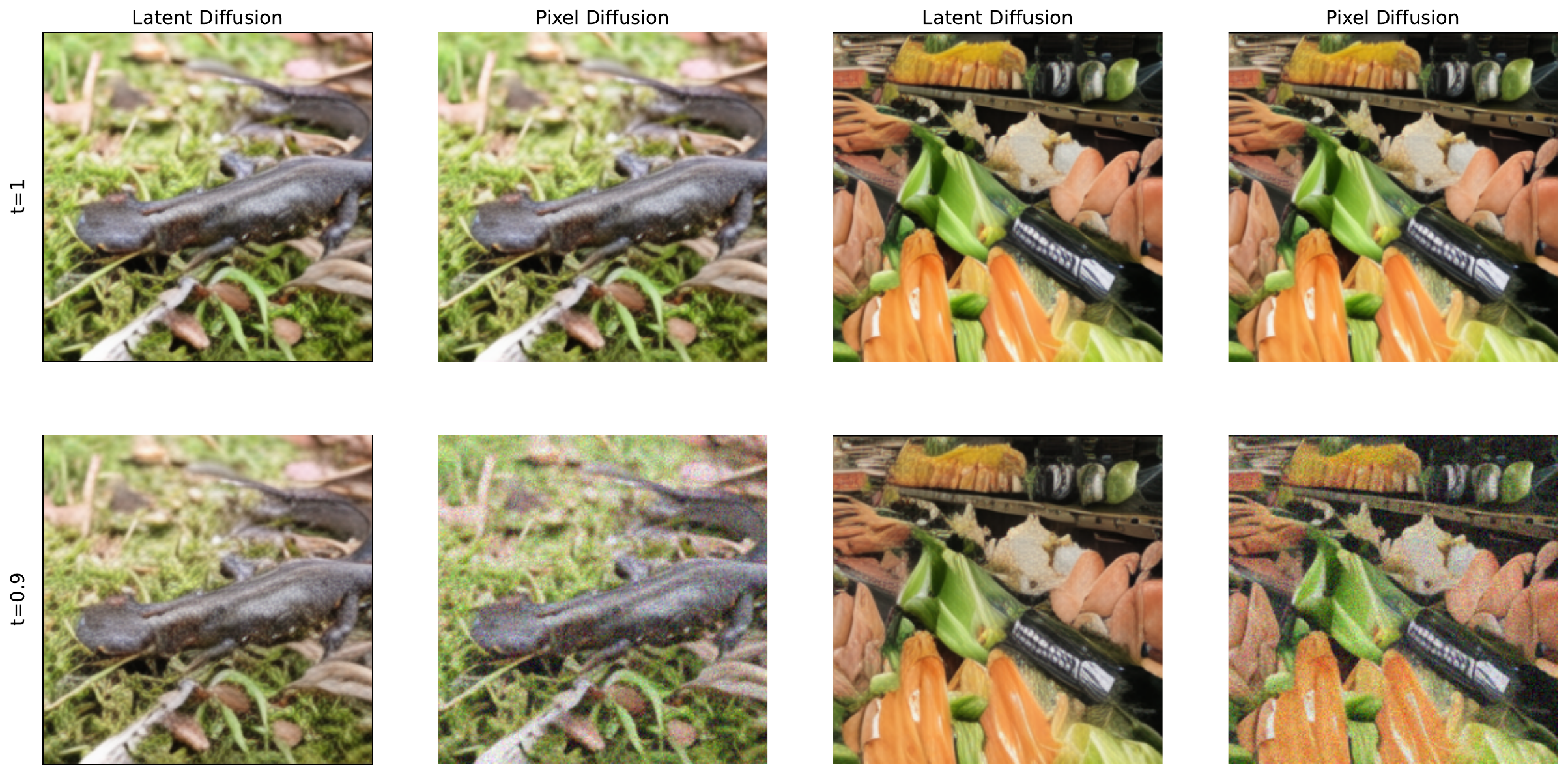}
    \caption{We compare qualitative results of image generated with LDM and pixel diffusion. The pixel diffusion images are generated by directly adding Gaussian noise to simulate the real pixel diffusion process.}
\end{figure}
\begin{figure}
    \centering
    \includegraphics[width=.7\linewidth]{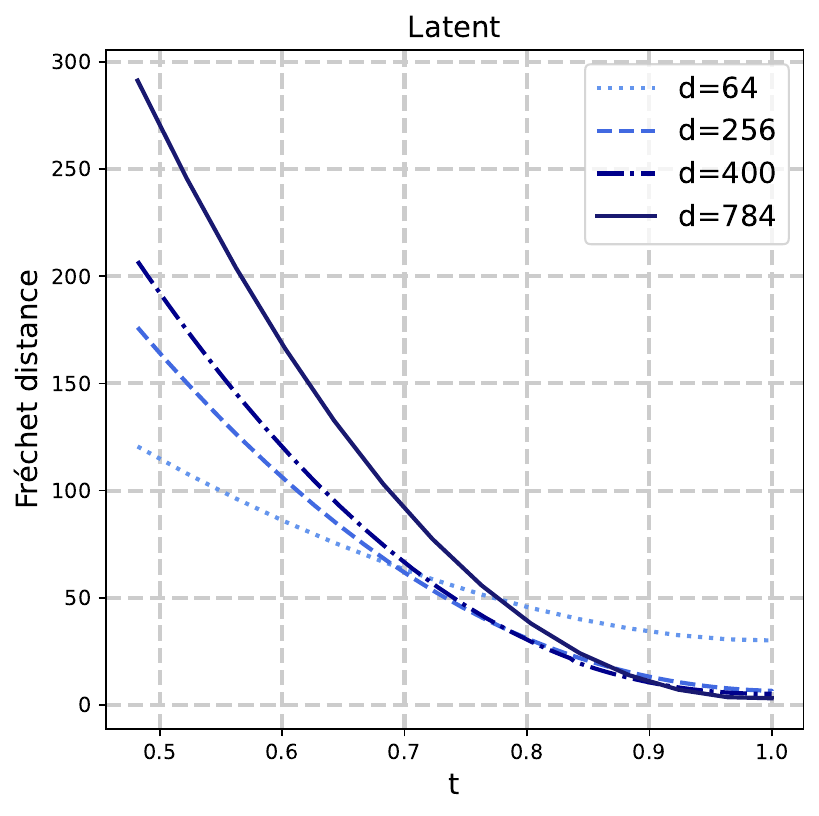}
    \caption{We train linear autoencoders on MNIST, then corresponding U-Net LDMs. The experiment shows a behavior similar to Proposition \ref{prop:min_wasserstein_projected} and Figure \ref{fig:projection_experiments}: for each LDM there exists a time interval such that the LDM is optimal, and for later diffusion times, a larger latent dimension is better.}
\end{figure}
\begin{figure}
    \centering
    \includegraphics[width=0.48\linewidth]{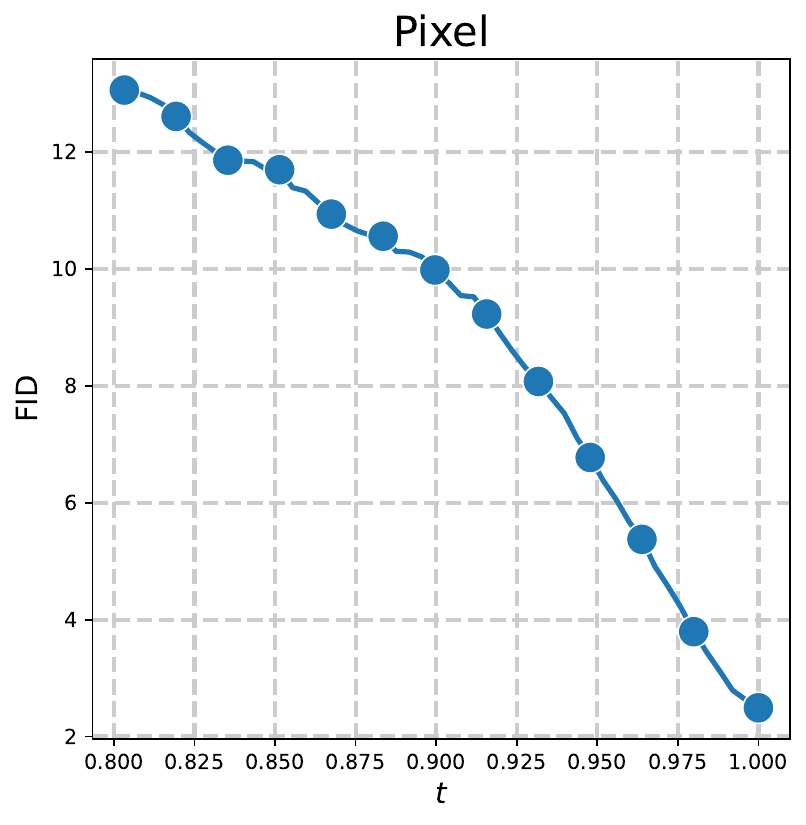}
    \includegraphics[width=0.48\linewidth]{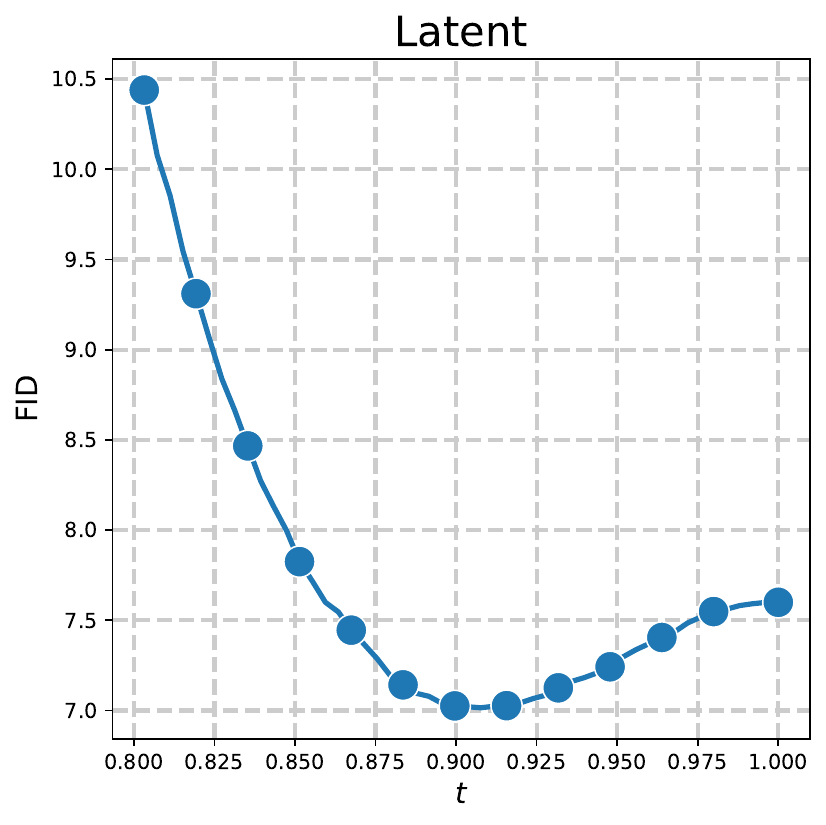}
    \caption{We train AE on CelebA-HQ and a corresponding LDM. In addition, we compare it to the FID curve of pixel diffusion model trained on CelebA. The experiment shows monotonicity in FID curve of pixel diffusion, while the FID curve of LDM achieves a minimum before the final diffusion time.}
    \label{fig:celeba-latent-pixel-fid-comparison}
\end{figure}
\begin{figure}[ht]
\centering
    \includegraphics[width=0.7\linewidth]{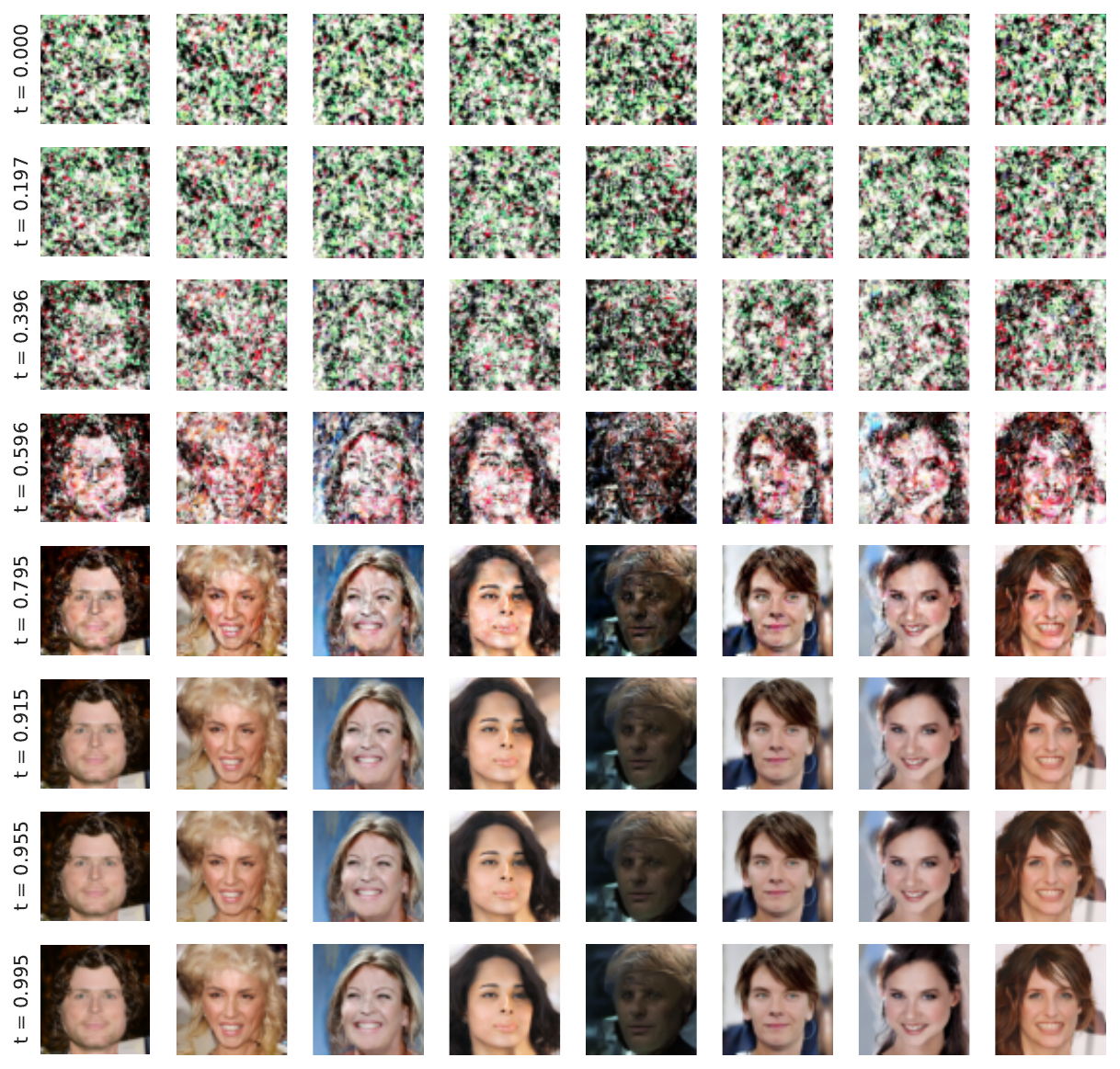}
    \caption{The final steps of LDM do not improve image quality.}
    \label{fig:latent-diffusion}
\end{figure}
\begin{figure}
    \centering
    \includegraphics[width=.7\linewidth]{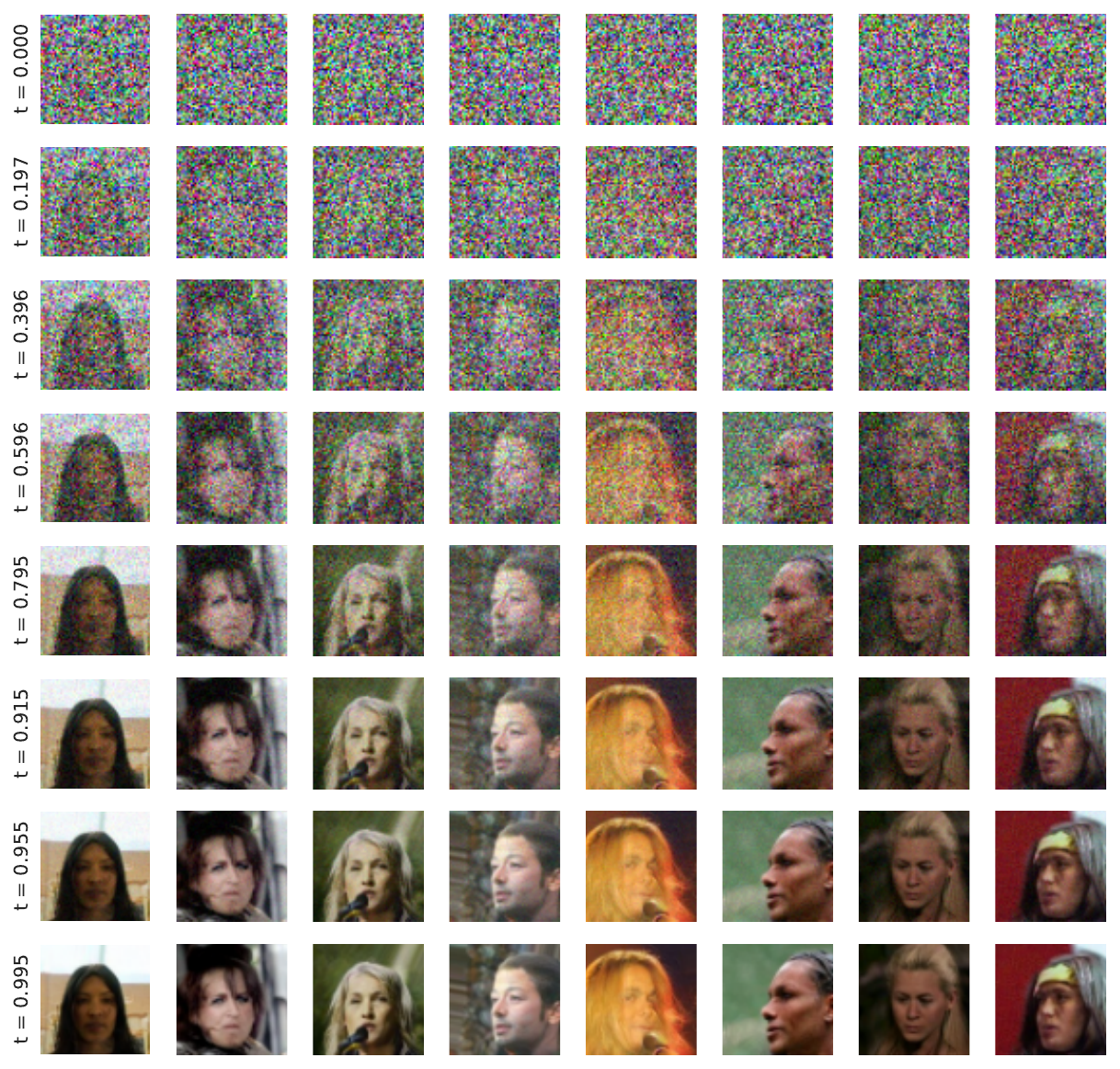}
    \caption{The quality of sample in diffusion on pixel space is still increasing in the final few steps.}
    \label{fig:pixel-diffusion}
\end{figure}

\end{document}